\documentclass[pdflatex,sn-mathphys-num]{sn-jnl}%

\usepackage{amssymb}
\usepackage{amsmath}
\usepackage{graphicx}
\usepackage{optidef}
\usepackage[ruled,vlined, linesnumbered, noend]{algorithm2e}
\usepackage{amsthm}
\usepackage{multirow}
\usepackage{diagbox}
\usepackage{adjustbox}
\usepackage{xcolor}
\usepackage{rotating}
\newtheorem{theorem}{Theorem}
\newtheorem{remark}{Remark}
\newtheorem{assumption}{\bf{Assumption}}
\newtheorem{corollary}{Corollary}

\usepackage{tabularx}  

\newtheorem{definition}{\bf{Definition}}
\usepackage{longtable}
\usepackage{booktabs}

\begin{document}

\title[Article Title]{RSPECT: Robust and Scalable Planner for Energy-Aware Coordination of UAV-UGV Teams in Aerial Monitoring}


\author[1]{\fnm{Cahit Ikbal} \sur{Er}}\email{er.c@northeastern.edu}

\author[2]{\fnm{Amin} \sur{Kashiri}}\email{kashiri.a@northeastern.edu}

\author*[1,2]{\fnm{Yasin} \sur{Yaz{\i}c{\i}o\u{g}lu}}\email{y.yazicioglu@northeastern.edu}

\affil[1]{\orgdiv{Department of Mechanical and Industrial Engineering}, \orgname{Northeastern University}, \city{Boston}, \postcode{02115}, \state{MA}, \country{USA}}

\affil[2]{\orgdiv{Department of Electrical and Computer Engineering}, \orgname{Northeastern University},  \city{Boston}, \postcode{02115}, \state{MA}, \country{USA}}

\abstract{We consider the robust planning of energy-constrained unmanned aerial vehicles (UAVs) and unmanned ground vehicles (UGVs), which act as mobile charging stations, to perform long-horizon aerial monitoring missions. More specifically, given a set of points to be visited by the UAVs and desired final positions of the UAV-UGV teams, the objective is to find a robust plan (the vehicle trajectories) that can be realized without a major revision in the face of uncertainty (e.g., unknown obstacles/terrain, wind) to complete this mission in minimum time.  We provide a formal description of this problem as a mixed-integer program (MIP), which is NP-hard. Since exact solution methods are computationally intractable for such problems,  we propose RSPECT, a scalable and efficient heuristic. We provide theoretical results on the complexity of our algorithm and the feasibility and robustness of resulting plans. We also demonstrate the performance of our method via simulations and experiments.}




\maketitle

\section{Introduction}
\label{sec:intro}
Unmanned aerial vehicles (UAVs) have proven to be invaluable in various domains such as precision agriculture, logistics, disaster response, security, and environment/infrastructure monitoring (e.g., \cite{tokekar2016sensor, manfreda2018uavmonitoring, ham2016visual, ren2019uavmonitoring, boccardo2015uav, asarkaya2021temporal}). One of the main roles of UAVs in such applications is to perform aerial monitoring, i.e., collecting sensor data (e.g., aerial images) from desired locations. In missions defined over large areas or long time-horizons, a major limitation of UAVs is their limited battery capacity. This limitation can be addressed by deploying some charging stations (e.g, \cite{Mulgaonkar2014AutonomousCharging, bacanli2021charging, Nigam2012ControlMultipleUAVs}). However, this approach requires establishing some infrastructure on the ground, which may be very inefficient or infeasible in some scenarios (e.g., missions over very large areas or in remote locations). Alternatively, the UAVs can be supported with unmanned ground vehicles (UGVs), which serve as mobile charging stations and transport the UAVs between different take-off/landing locations. Recently, there has been a growing interest in energy-aware planning of UAV-UGV teams. Coordination of UAVs with UGVs for monitoring/coverage tasks intersects with several research areas.

In transportation and logistics, problems similar to the UAV-UGV cooperative routing problem are studied as Vehicle Routing Problem with Drones (VRP-D) or Traveling Salesperson Problem with Drones (TSP-D) \cite{wang2019vehicle, tang2019study}. In \cite{murray2015flying}, a Mixed Integer Linear Programming (MILP) formulation was investigated for a similar problem regarding last-mile delivery with truck-drone system. A MILP formulation and a heuristic approach based on TSP-LS was studied in \cite{ha2018min}. In \cite{bouman2018dynamic}, an exact solution approach for TSP-D was provided with dynamic programming. In \cite{chen2021delivery}, a heterogeneous robot planning problem for urban parcel delivery was explored, where UGVs were navigating in a road network. Binary integer programming models for two-echelon cooperative routing were introduced in \cite{luo2017two, li2021ground}. A two-stage, route-focused framework with a truck and a drone was devised in \cite{liu2020two}, optimizing both vehicle routes via hybrid heuristics. While these works show some similar applications to aerial-ground cooperation, they focus on delivery and logistics  rather than aerial monitoring.

In robotics, cooperative routing problems involving single UAV-UGV systems were addressed in \cite{maini2015cooperation, manyam2019cooperative}, using greedy heuristics, MILP formulations, and Branch and Cut algorithms. In \cite{ropero2019terra}, given a set of points to be visited by the UAV, first a set of rendezvous locations (charging stops) are computed based on Voronoi tessellations, then the UAV and UGV trajectories are computed separately via a genetic algorithm and a variant of $A^*$. In \cite{yu2018algorithms}, an algorithm was proposed that determines the UAV path and the optimal locations for landing on stationary or mobile UGV-based charging stations by formulating the problem as a variant of the Generalized Traveling Salesperson Problem (GTSP). In \cite{maini2019coverage}, a two-stage heuristic for UAV-UGV route planning with energy constraints was proposed, leveraging a mobile ground vehicle constrained to a road network as a refueling station to maximize coverage. 

Coordinating multiple UAV-UGV teams relates to multi-agent task allocation and the multiple Traveling Salesperson Problem (mTSP), where multiple agents collectively visit locations while minimizing total cost. In \cite{seyedi2019persistent}, a scalable heuristic for persistent surveillance was presented, partitioning the environment into smaller regions that can be covered in a single tour by the energy-limited UAVs, which are recharged and transported between those regions by the UGVs. Hierarchical, bi-level optimization frameworks for routing of multiple energy-aware UAVs and a single UGV have been explored in \cite{ramasamy2021cooperative, ramasamy2022coordinated, mondal2023optimizing}, typically employing clustering methods  (e.g., K-means) to determine UGV visits and TSP/VRP for route planning. A sequential planning framework was proposed in \cite{mondal2025cooperative}, using minimum set cover for UGV stops and energy-constrained vehicle routing for UAV paths. A deterministic multi-agent routing framework for heterogeneous UAV-UGV teams with asynchronous planning was proposed in \cite{mondal2024robust}, addressing constraint satisfaction across varying team compositions.

A common limitation across most of these approaches is the assumption of complete environmental knowledge. When unknown obstacles are encountered or travel times exceed predictions due to disturbances, such methods may become infeasible and require complete re-planning—a computationally expensive process. Few works explicitly address uncertainty in UAV-UGV coordination. In \cite{shi2022risk}, linear programming was used to find optimal rendezvous points to modify the trajectories of a UGV (charging station) and an energy-limited UAV with stochastic fuel consumption. In \cite{mondal2025risk}, a risk-aware reinforcement learning framework was proposed, accounting for stochastic energy consumption, enforcing empirical risk constraints during policy optimization. In \cite{Lin2022}, building on \cite{seyedi2019persistent}, scalable and robust planning of UAV-UGV teams in persistent surveillance problems in the face of a priori unknown obstacles was achieved. While these approaches represent important progress in robust planning, they either lack formal guarantees on mission feasibility or rely on environment partitioning strategies. Our work provides a complementary approach with formal robustness guarantees and addresses scalable multi-team coordination, an area that has received limited attention in robust UAV-UGV planning.

\color{black}
In this paper, we propose a scalable and efficient algorithm, named \textbf{R}obust and \textbf{S}calable \textbf{P}lanner for
    \textbf{E}nergy-Aware \textbf{C}oordination of UAV-UGV \textbf{T}eams in Aerial Monitoring (RSPECT), for robust planning of UAV-UGV teams to perform long-horizon aerial monitoring missions. \color{black} In particular, we consider missions where a given set of aerial locations must be visited by a collection of UAV-UGV teams, each of which consist of  an energy-limited UAV and a supporting mobile charging station (UGV). Each team starts from an initial location and must end the mission at a specified final location. The goal is to achieve robust offline planning of the vehicles to minimize the overall mission completion time while ensuring that all aerial locations will be visited by one of the UAVs without violating the UAVs' energy constraints, even in the face of bounded disturbances that may cause longer travel times or deviations from the offline plan. The main contributions of this paper are as follows:

\begin{enumerate}
    \item We formulate a novel robust, energy-aware, multi UAV-UGV planning problem as a mixed-integer program for aerial monitoring missions.
    \item We propose a polynomial-time and efficient heuristic, RSPECT (Alg.\ref{alg:integrated_path_planning}), to solve this NP-hard planning problem.
    \item We present theoretical results on the computational complexity (Theorem \ref{th-1}) and robustness (Theorem \ref{th2}, Corollary \ref{cor}) of our algorithm. 
    \item We demonstrate the performance of our approach via simulations, numerical comparisons with standard optimization techniques (branch and cut, simulated annealing, genetic algorithm) and other heuristics proposed in closely related works \cite{ropero2019terra,yu2018algorithms,maini2019coverage}, and real-world experiments.
\end{enumerate}
\color{black}

\section{Problem Formulation}
\label{sec:problem_formulation}

 An aerial monitoring mission is being considered, where small UAVs are supported by UGVs, which serve as mobile charging stations and transport the UAVs between take-off/landing locations. 
An overview of the representative multi-team UAV–UGV aerial monitoring mission environment considered in this work is illustrated in Fig. \ref{fig:concept}
\begin{figure}[htpb]
  \centering
  \includegraphics[width=1.0\linewidth]{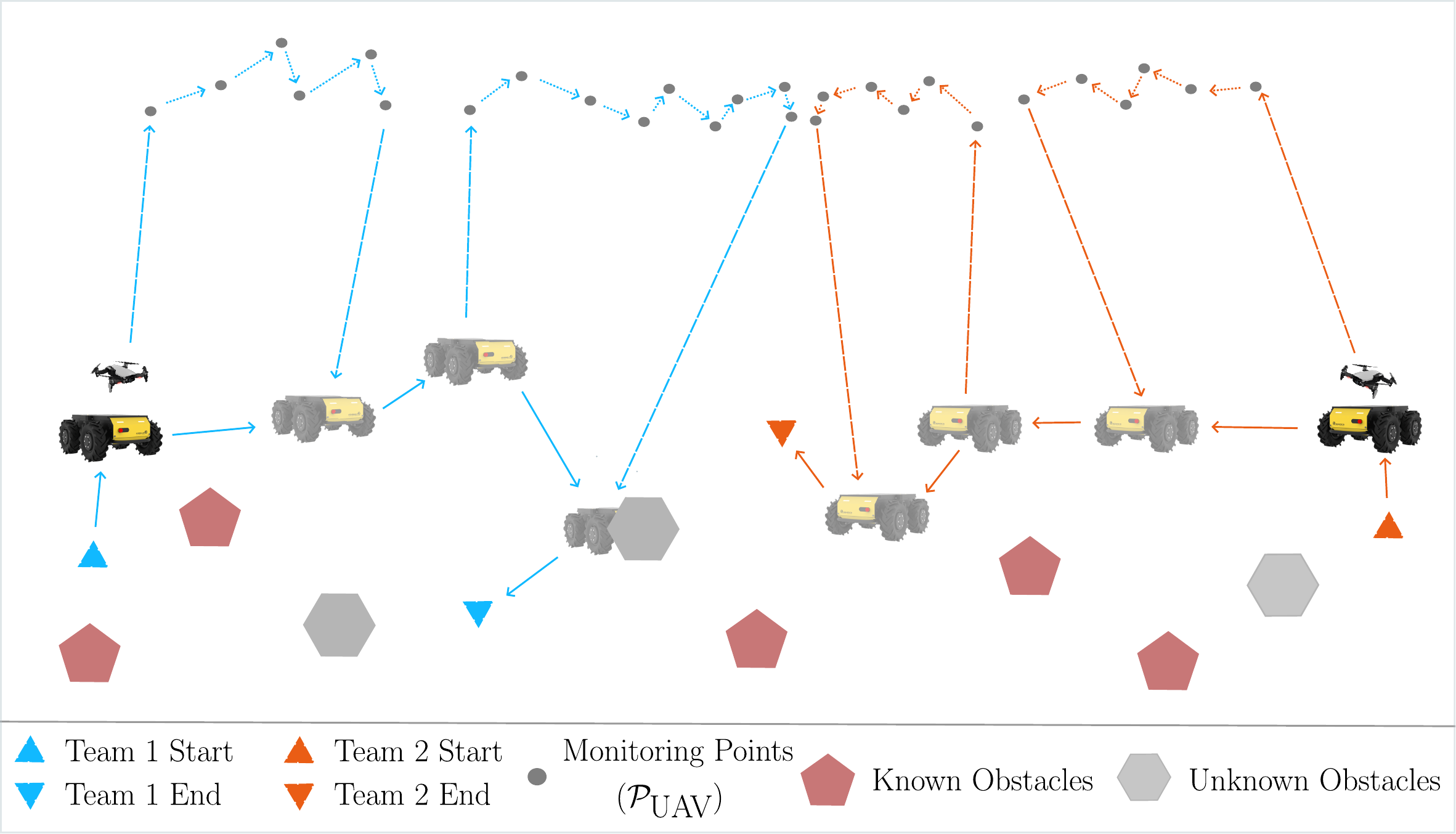}
    \caption{Each team (UAV-UGV pair) has given start/end positions. The goal is to achieve robust offline planning of the UAV-UGV teams to minimize the mission completion time while ensuring all aerial monitoring points are visited by one of the UAVs without violating the UAVs' energy constraints, even under disturbances that cause longer travel times or deviations from the plan.}
    \label{fig:concept}    
\end{figure}

In the remainder, $\mathbb{R}$ and $\mathbb{R}_{\geq 0}$ will be used to denote the sets of real numbers and non-negative real numbers, respectively. Let the environment be represented as 
\begin{equation}
\label{env}
    \mathcal{Q}{:}{=}\{(x,y,z) \in \mathbb{R}^{3} | \,\, 0 \leq x \leq \bar{x},\,0 \leq y \leq \bar{y},\,
 0 \leq z \leq \bar{z}\},
\end{equation}
\textcolor{black}{which contains a feasible space $\mathcal{Q}_f \subseteq \mathcal{Q}$ determined by the obstacles that are known a priori. For example, if none of the obstacles are known, then $\mathcal{Q}_f = \mathcal{Q}$.} Furthermore, let $\mathcal{Q}^g_f$ and $\mathcal{Q}^a_f$ denote the set of feasible points on the ground and in the air, respectively, i.e.,  $\mathcal{Q}^g_f=\{(x,y,z)\in \mathcal{Q}_f |  z=0\}$, and $\mathcal{Q}^a_f=\mathcal{Q}_f \setminus \mathcal{Q}^g_f$. Both \( \mathcal{Q}_f \) and its subset on the ground \( \mathcal{Q}_f^g \) are assumed to be connected, meaning that there exists a feasible path between any pair of points within the respective sets. This is just a mild assumption to ensure the feasibility of the monitoring mission, as otherwise part of the environment would be inaccessible to the vehicles. The function $\ell:\mathcal{Q}_f \times \mathcal{Q}_f \rightarrow \mathbb{R}_{\geq 0}$ is defined as a function that returns the distance (length of the shortest path) between any two points in the feasible space, $\mathcal{Q}_f$. The overall mission goal is for the teams to collectively visit a set of $n$ points in the air, $\mathcal{P}_{\text{UAV}}= \{p_1, \dots, p_n\} \subset \mathcal{Q}_f^a$.\footnote{\color{black}The formulation makes no assumption on how $\mathcal{P_\text{UAV}}$ is generated. The points may arise from arbitrary mission requirements such as imaging regions of interest from specific viewpoints with a desired resolution. } Let there be $m$ teams (UAV–UGV pairs) indexed by $\mu = \{1,\ldots,m\}$, where $m\leq n$. Each team starts from its initial position $p_{\text{o}}^{\mu} \in \mathcal{Q}_f^g$ and must reach to its specified final position $p_{\text{f}}^{\mu} \in \mathcal{Q}_f^g$. Collectively, the teams must ensure that all points in $\mathcal{P}_{\text{UAV}}$ are visited without violating UAVs' energy constraints, while minimizing completion time.

During this mission, the UAVs should never be expected to continuously fly longer than a specified duration due to their energy limitation. \textcolor{black}{For UAVs where the majority of the mission is flown at approximately constant average speed, power consumption is approximately constant \cite{abdilla2015power, stolaroff2018energy}, so energy consumed is proportional to flight time. Under this assumption, maximum flight time directly represents the UAV's energy budget i.e., the flight time at which the battery is exhausted. This endurance-based modeling is standard in the UAV-UGV planning literature \cite{ropero2019terra,yu2018algorithms,maini2019coverage}.}

In contrast, UGVs have substantially larger onboard batteries and are assumed to have sufficient energy for the mission. To satisfy the UAVs' energy requirements, we focus on an approach where the UGVs transport the UAVs through a sequence of points, which are called the \emph{release} and \emph{collect points}. At each release point, the UAV takes-off and visits a subset of $\mathcal{P}_{\text{UAV}}$, then lands on the UGV at the corresponding collect point to be recharged and transported to the next release point (or the respective final position if no more points from $\mathcal{P}_{\text{UAV}}$ will be visited by the team). Clearly, a necessary condition for the optimality of such a solution is that the UAV samples at least one point from $\mathcal{P}_{\text{UAV}}$ in each tour. Furthermore, each tour can be defined by at most $n+2$ points (a release and a collect point, and at most $n$ points from $\mathcal{P}_{\text{UAV}}$). Accordingly, the mission plan of each team $\mu \in \{1,2,\hdots, m\}$ is encoded as a matrix:
\begin{equation*}
    {X}^{\mu} = \begin{bmatrix}
        {X}^{\mu}_{1,1} & \cdots & {X}^{\mu}_{1,n+2} \\
        \vdots & \ddots & \vdots \\
        {X}^{\mu}_{n,1} & \cdots & {X}^{\mu}_{n,n+2}
    \end{bmatrix},
\end{equation*}
where each \({X}^{\mu}_{i,j} \in \mathcal{Q}_f\) denotes a release point if $j= 1$, a collect point if $j= n+2$,  or a point to be visited by the UAV if $ j \in \{2, \hdots ,n+1\}$. Each row $i$ of ${X}^{\mu}$ (denoted by ${X}^{\mu}_{i}$) is used to encode a tour, i.e., the UAV traversing a shortest path along the following sequence of waypoints in that row and meeting with the UGV at ${X}^{\mu}_{i,n+2}$. Using this, any plan with fewer than $n$ tours can be expressed by having ${X}^{\mu}_{i,1}{=}{X}^{\mu}_{i,2}{=}\hdots{=}{X}^{\mu}_{i,n{+}2}$ in certain rows of ${X}^{\mu}$ (indicating a ``trivial tour" of zero length).
After each tour $i$ is completed, the UGV transports the UAV from the collect point ${X}^{\mu}_{i,n+2}$ to the next release point ${X}^{\mu}_{i+1,1}$ (or the specified final position $p^{\mu}_{\text{f}}$ if $i=n$) while recharging it. For each team $\mu$, the total time to execute the respective plan ${X}^{\mu}$ is 

\begin{equation}
\begin{aligned}
\tau({X}^{\mu}) ={} & \tau_g(p^{\mu}_{\text{o}}, {X}^{\mu}_{1,1}) + \tau_g({X}^{\mu}_{n,n+2}, p^{\mu}_{\text{f}}) \\
& {}+ \sum_{i=1}^{n} \max\!\big(\tau_a({X}^{\mu}_{i}), \tau_g({X}^{\mu}_{i,1}, {X}^{\mu}_{i,n+2})\big) \\
& {}+ \sum_{i=1}^{n-1} \max\!\big(\tau_g({X}^{\mu}_{i,n+2}, {X}^{\mu}_{i+1,1}), \tau_c({X}^{\mu}_{i})\big)\,,
\end{aligned}
\label{eq:objective_function}
\end{equation}
\color{black}
where $\tau_a:\mathcal{Q}_f^{n+2} \rightarrow \mathbb{R}_{\geq 0}$ and $\tau_g:\mathcal{Q}_f \times \mathcal{Q}_f \rightarrow \mathbb{R}_{\geq 0}$ are functions denoting the time to traverse the corresponding paths by the UAV and the UGV \textcolor{black}{within the feasible space $\mathcal{Q}_f$}, respectively. By definition, $\tau_a$ and $\tau_g$ equal zero for zero-length paths.\footnote{We do not provide any specific functions $\tau_a$ and $\tau_g$ in the formulation to keep it generic. Our approach can accommodate arbitrary choices of $\tau_a$ and $\tau_g$. One simple example is setting the travel time equal to the length of respective shortest path divided by a constant speed.}  Each $\tau_a({X}^{\mu}_{i})$ denotes the time required for the  $\mu^{\text{th}}$ UAV to complete the $i^{\text{th}}$ tour, starting from its release point ${X}^{\mu}_{i,1}$, traversing the following sequence of points in the tour, ${X}^{\mu}_{i,2}, \hdots, {X}^{\mu}_{i,n+1}$, and meeting with the $\mu^{\text{th}}$ UGV  at the collect point ${X}^{\mu}_{i,n+2}$. The term $\tau_g(p^{\mu}_{\text{o}}, {X}^{\mu}_{1,1})$ represents the time for the $\mu^{\text{th}}$ UGV to travel from its start position $p^{\mu}_{\text{o}}$ to the first release point, whereas $\tau_g({X}^{\mu}_{n,n+2}, p^{\mu}_{\text{f}})$ denotes the time for the $\mu^{\text{th}}$ UGV to travel from the final collect point to $p^{\mu}_{\text{f}}$. 
Similarly, $\tau_g({X}^{\mu}_{i,n+2}, {X}^{\mu}_{i+1,1})$ captures the time spent by the UGV while moving between the collect and release points of consecutive tours $i$ and $i+1$, during which the UAV is recharged for the next flight. The third term of the summation in \eqref{eq:objective_function} is the total time spent by team $\mu$ to complete the specified tours,
and the last term is the total time spent between the end and start of consecutive tours of team $\mu$. Here, for each tour $i$, $\tau_{\text{g}}({X}^{\mu}_{i,1}, {X}^{\mu}_{i,n+2})\geq0$ is the time for UGV to go from \({X}^{\mu}_{i,1}\) to \({X}^{\mu}_{i,1n+2}\) over a shortest path. Similarly,  $\tau_{\text{g}}({X}^{\mu}_{i,n+2}, {X}^{\mu}_{i+1,1})\geq0$ is the time for the $\mu^{\text{th}}$ UGV to transport the $\mu^{\text{th}}$ UAV from \({X}^{\mu}_{i,n+2}\) to \({X}^{\mu}_{i+1,1}\) over a shortest path. The time to fully recharge\footnote{As an alternative to recharging, a battery-swapping approach can also be modeled like this by setting $\tau_{\text{c}}({X}^{\mu}_{i})$ equal to a small constant.} the UAV after it completes the tour $i$ is denoted by $\tau_{\text{c}}({X}^{\mu}_{i})\geq0$. Since UAVs start the mission fully charged, this ensures that they are fully charged before each tour. 
We define our robust planning problem as minimizing the maximum mission time among $m$ teams under the mission constraints, i.e.,

\color{black}
\begin{mini!}|s|[2] 
    {\{X^{\mu}\}_{\mu{=}1}^{m}}                        
    {\displaystyle \max_{\mu = {\{1,\dots,m}\}} \, \tau(X^{\mu}) \label{eq:eq1}}   
    {\label{opt_problem}} 
    {}
    \addConstraint{\forall p_k \in \mathcal{P}_{\text{UAV}},\, \exists\, X^{\mu}_{i,j} = p_k}{ \label{eq:con1}}    
    \addConstraint{\tau_{\text{a}}(X^{\mu}_{i}) + \delta_{a} \leq \overline{\tau_{a}},}{\forall i, \forall \mu \label{eq:con2}}
    \addConstraint{\tau_{g}(X^{\mu}_{i,1}, X^{\mu}_{i,n+2}) + \delta_{g} \leq \overline{\tau_{a}},}{\forall i, \forall \mu \label{eq:con3}}
    \addConstraint{X^{\mu}_{i,1},\, X^{\mu}_{i,n+2} \in \mathcal{Q}_f^g,}{\forall i, \forall \mu \label{eq:con4}}
    \addConstraint{X^{\mu}_{i,j} \in \mathcal{P}_{\text{UAV}} \cup \{X^{\mu}_{i,j-1}\},}{\forall j \in \{2, {\hdots}, n{+}1\}, \forall i, \forall \mu, \label{eq:con5}}
\end{mini!}
\color{black}
where $\tau(X^{\mu})$ is the total time for team $\mu$ to complete its plan as given in \eqref{eq:objective_function}, and the constraints are as follows: First, \eqref{eq:con1} ensures that all points in \(\mathcal{P}_{\text{UAV}}\) are visited by at least one team. Second,  \eqref{eq:con2} implies that none of the tours require the respective UAV to fly longer than $\overline{\tau_{\text{a}}}$, even if the time to complete the tour increases by some specified $\delta_a \geq 0$ during the mission due to disturbances. In any tour, if the UAV arrives at the collect location earlier than the UGV, it needs to hover until the UGV's arrival.
Accordingly, \eqref{eq:con3} ensures that, for any team, the UGV arrives at the collect location before the UAV runs out of energy, even if the ground travel time increases by $\delta_g \ge 0,$ playing a similar role for the ground leg as $\delta_a$ does for the aerial leg.   Here, the choice of robustness parameters, $(\delta_a,\delta_g)$\footnote{\color{black}The parameters define interval uncertainty sets for possible increase in the UAV and UGV travel times, $[0,\delta_a]$ and $[0,\delta_g]$, respectively. Constraints \eqref{eq:con2}, \eqref{eq:con3} enforce mission feasibility for all realizations of uncertainty from these sets. This formulation thus follows the deterministic worst-case robust optimization paradigm under interval uncertainty~\cite{ben2009robust, bertsimas2011theory}.}, determines how much increase in travel times due to disturbances (e.g., uncertain terrain/wind, avoiding collisions with unknown obstacles or other robots\footnote{\textcolor{black}{We do not introduce any inter-robot collision avoidance constraint in \eqref{opt_problem} since it can be handled during execution with negligible increase in travel times for long-horizon missions in large environments.}}) can be tolerated during the mission execution without needing a major revision (changing the tours taken by the UAV) to the resulting plan. 
Fourth, for any team, \eqref{eq:con4} ensures that all the release and collect points are feasible points on the ground. Finally, \eqref{eq:con5} ensures that, in any row $i$ of ${X}^{\mu}$, each entry other than the first and the last ones is either a point from $\mathcal{P}_{\text{UAV}}$ or just a repetition of the previous entry. Accordingly, either all those points are equal to the respective release point ${X}^{\mu}_{i,1}$ (e.g., a trivial tour) or there exists some ${X}^{\mu}_{i,j} \neq {X}^{\mu}_{i,1}$ (i.e., the UAV departs from $X^{\mu}_{i,1}$) such that all the following entries in that row until the collect point are from $\mathcal{P}_{\text{UAV}}$, i.e., ${X}^{\mu}_{i,j'} \in \mathcal{P}_{\text{UAV}}$ for all $j' \in \{j, \hdots, n+1\}$.  Note that \eqref{eq:con5} does  not impose any restriction on how many points from $\mathcal{P}_{\text{UAV}}$ can be included in each tour since the respective entries, ${X}^{\mu}_{i,j}$, are not necessarily distinct.
The overall mission plan is denoted as 
$X = \{ X^{1}, \hdots, X^{m}  \}$,
representing the collection of all feasible team plans, where team $\mu$'s plan is denoted by $X^{\mu}$.



\begin{definition} (Tour Robustness) For any feasible mission plan $X = \{ X^{1}, \hdots, X^{m} \}$, the robustness of the $i^{\text{th}}$ tour of the $\mu^{\text{th}}$ team is given by the maximum increases in the travel times of the UAV/UGV in that tour that do not violate UAVs' energy limitations encoded in \eqref{eq:con2},\eqref{eq:con3}:
\begin{equation*}
\hat{\delta}^{\mu, i}_a(X)
=  \overline{\tau_{a}} - \tau_a(X^{\mu}_{i}), 
\end{equation*}
\begin{equation*}
\hat{\delta}^{\mu,i}_g(X)
= \overline{\tau_{a}} - \tau_g(X^{\mu}_{i,1}, X^{\mu}_{i,n+2}).
\end{equation*}
\label{def1}
\end{definition}
For any feasible~$X$, since (\ref{eq:con2}) and (\ref{eq:con3}) are satisfied for every team $\mu$ and every tour $i$, $\delta_a$ and $\delta_g$ in~\eqref{opt_problem} impose lower bounds on the robustness of every tour, i.e., $\delta^{\mu,i}_a(X) \geq \delta_a$ and 
$\delta^{\mu,i}_g(X) \ge \delta_g$ for every $i$ and $\mu$.  Accordingly, the robustness parameters $\delta_a$ and $\delta_g$ in \eqref{opt_problem} can be tuned to achieve a desired trade-off between robustness and efficiency. Selecting larger values of $\delta_a$ and $\delta_g$ leads to a plan $X$ with more robust tours, but it also reduces the feasible space of the planning problem, potentially leading to solutions with higher cost mission times. \textcolor{black}{For example, if many/large unknown ground obstacles are anticipated, a larger $\delta_g$ can be selected, which effectively forces the release and collect points to be closer together for caution. Analogously, it would be desirable to select $\delta_a$ as the maximum possible increase in the UAVs' travel times during tours, accounting for both potential changes in the release/collect points due to obstacles and in-flight disturbances (obstacles, wind, etc.).}


\section{Proposed Solution}
\label{sec:proposed_solution}
The robust planning problem given in \eqref{opt_problem} is a mixed-integer program (MIP), which is NP-Hard. Hence, exact solution methods (e.g., Branch \& Cut \cite{mitchell2002branch}) are computationally intractable with large $|\mathcal{P}_{\text{UAV}}|$. We propose a scalable and efficient heuristic to overcome this challenge. In this section, we present our approach and its \textcolor{black}{formal guarantees.\footnote{\color{black} The formal guarantees provided here concern feasibility and worst-case time complexity (Theorem \ref{th-1}) and robustness (Theorem \ref{th2}, Corollary \ref{cor}). No approximation ratio or bound on solution quality relative to the optimum is established. Solution quality is assessed empirically in Section \ref{sec:simulations}.}} 
\subsection{Proposed Robust Planning Algorithm}
\label{subsec:proposed_algorithm}
Our proposed  algorithm, Alg.\ref{alg:integrated_path_planning}, uses the fact that \eqref{opt_problem} share structural similarity with multiple Traveling Salesperson Problem (mTSP). The mTSP is a generalization of TSP, where multiple agents collectively visit all nodes while minimizing overall cost. In our setting, there are additional constraints that stem from the limitations of UAV-UGV teams. Our approach consists of 2 steps, as illustrated in Fig. \ref{fig:alg1}. 

\begin{figure*}[htpb]
  \centering
  \includegraphics[width=1.0\linewidth]{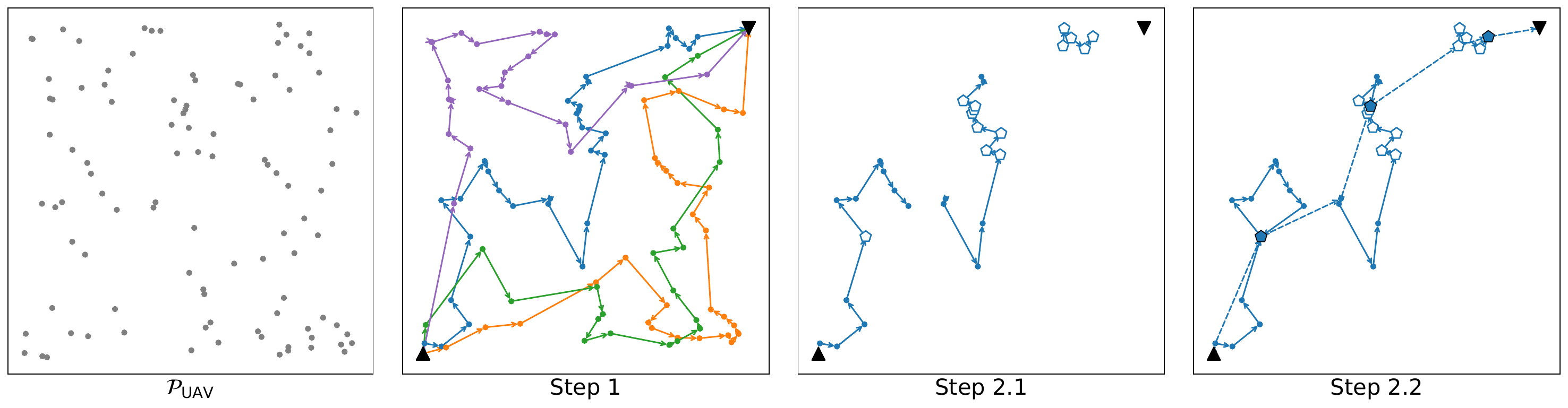}
    \caption{\color{black}Output of Alg.~\ref{alg:integrated_path_planning} for $m=4$ teams 
(Top View). Left: $\mathcal{P}_{\text{UAV}}$ (gray points). In Step~1, monitoring points are partitioned among the $m$ teams and a visit sequence is constructed for each team; the resulting sequences are shown as color-coded polylines, with one color per team. Upward and downward triangles show $p_{\text{o}}^{\mu}$ and $p_{\text{f}}^{\mu}$; in this example, all teams share the same 
start and end positions (both shown in black). Steps 2.1-2.2 show the detailed planning for one of the teams (blue, $\mu=1$). In Step~2.1, blue lines are feasible tours for the UAV-UGV team, 
and hollow pentagons are feasible collect points for each tour. In Step~2.2, the plan is obtained: filled pentagons represent the selected collect point for each tour, and dashed blue lines are the UGV path.}
    \label{fig:alg1}    
\end{figure*}

\color{black}
In Step~1, we partition $\mathcal{P}_{\text{UAV}}$ among the $m$ teams while simultaneously constructing an efficient visit sequence for each team. The algorithm proceeds iteratively: at each iteration, an unassigned monitoring point $p$ is selected uniformly at random from the set of remaining points $\mathcal{P}_{\text{rem}}$ (line~10). For each team $\mu$, a candidate TSP path $\mathcal{T}_{\mu}$ is computed over $\mathcal{P}_{\text{UAV}}^{\mu} \cup \{p\}$, where $\mathcal{P}_{\text{UAV}}^{\mu}$ denotes the set of points assigned to team $\mu$ so far, with $p_o^{\mu}$ and $p_f^{\mu}$ as fixed start and end points. This is achieved by using a Nearest-Neighbor TSP heuristic (NN-TSP) \cite{rosenkrantz1977analysis} and the resulting path cost $D'_{\mu}$ is recorded (line~12). The point $p$ is then assigned to the team $\mu^{\star}$ whose addition minimizes the resulting maximum path cost across all teams (line~13):
\begin{equation}
\mu^{\star} \gets \mathop{\mathrm{arg\,min}}_{\mu \in \{1,\dots,m\}} 
\max\big(\{D_{\nu} \mid \nu \neq \mu\} \cup \{D'_{\mu}\}\big),
\label{eq:minmax_assignment}
\end{equation}
where $D_{\nu}$ denotes the current TSP path cost of team $\nu$ based on its assigned points $\mathcal{P}_{\text{UAV}}^{\nu}$.  After assignment, $p$ is added to $\mathcal{P}_{\text{UAV}}^{\mu^{\star}}$, the team's sequence $\mathcal{S}^{\mu^{\star}}$ is updated to the new path (with the $p_o^{\mu^{\star}}$ and $p_f^{\mu^{\star}}$ removed, so that $\mathcal{S}^{\mu^{\star}}$ only contains the air points $\mathcal{P}_{\text{UAV}}^{\mu^{\star}}$), and $D_{\mu^{\star}}$ is updated (line~14). This process repeats until all points are assigned. The output of Step~1 is the collection of sequences $\{\mathcal{S}^{\mu}\}_{\mu=1}^{m}$, where each $\mathcal{S}^{\mu}$ is a TSP-ordered permutation of $\mathcal{P}_{\text{UAV}}^{\mu}$. For the special case $m=1$, this procedure is replaced by a single TSP instance over $\mathcal{P}_{\text{UAV}}$ as given in line 4.

\color{black}
In Step~2, we construct a feasible plan for each team $\mu$ based on its assigned sequence $\mathcal{S}^{\mu}$ from Step~1. The algorithm iterates over all teams $\mu \in \{1,\ldots,m\}$ and, for each team, constructs $X^{\mu}$. This process has 2 sub-steps:

In Step~2.1, we obtain a feasible plan $X^{\mu}$ for team $\mu$, (the collection of all such plans, $X = \{X^{1}, \hdots, X^{m}\}$, constitutes a feasible solution to \eqref{opt_problem}) by breaking the sequence $\mathcal{S}^{\mu}$ from Step~1 into maximal tours that are feasible both for the $\mu^{\text{th}}$ energy-limited UAV and the $\mu^{\text{th}}$ UGV. Starting with $i=1$, each row $i$ of $X^{\mu}$ is iteratively built. If all the points in $\mathcal{S}^{\mu}$ have already been added to $X^{\mu}$, no further tours are needed, and $X^{\mu}_{i,1}$ is set to $X^{\mu}_{i-1,1}$ (lines~23-24). Otherwise, $X^{\mu}_{i,1}$ is set to $\pi(\mathcal{S}_k^{\mu}) = \mathop{\mathrm{arg\,min}}_{p \in \mathcal{Q}_f^g} \ell(p, \mathcal{S}_k^{\mu})$, the closest feasible point on the ground to $\mathcal{S}_k^{\mu}$, i.e., the next point in $\mathcal{S}^{\mu}$ (lines~25-26). The collect point of the current tour is set as the release point, serving as a valid default structure (line~27). For the remaining entries in the row, i.e., each $X^{\mu}_{i,j}$ for $n_{\mu}+1 \geq j \geq 2$, if all the points in $\mathcal{S}^{\mu}$ have already been visited ($k = n_{\mu}+1$), then $X_{i,j} = X_{i,j-1}$, which implies no motion (line~39). Otherwise (lines~29-37), $\mathcal{S}_k^{\mu}$ is added to the current tour, and a set of feasible collect points, $C$, is created. Specifically, for each candidate collect point $\pi(X^{\mu}_{i,r})$ ($2 \leq r \leq j$), we check:
\begin{enumerate}
    \item Whether the UAV can complete its tour and land at $\pi(X^{\mu}_{i,r})$ before violating its battery constraint, \eqref{eq:con2}.
    \item Whether the UGV can meet the UAV at $\pi(X^{\mu}_{i,r})$ before the UAV's battery is depleted, \eqref{eq:con3}.
\end{enumerate}
Note that $\pi(X^{\mu}_{i,2}) = X^{\mu}_{i,1}$. Therefore, all the points in the current tour are checked as possible collect points. If both conditions are satisfied, $\pi(X^{\mu}_{i,r})$ is added to $C$. If $C$ is non-empty, there is at least one feasible collect point. In that case, $C$ is stored as $C^{\mu}_i$, the set of feasible collect points for the $i^{\text{th}}$ tour. Otherwise, $\mathcal{S}_k^{\mu}$ is removed from the tour and the previous point is copied.

In Step~2.2, we select the best collect point for each tour by solving a Generalized Traveling Salesman Problem (GTSP). The objective is to minimize team $\mu$'s plan execution time, $\tau(X^{\mu})$. GTSP extends TSP by enforcing the agent to visit exactly one node from each cluster. Here, the clusters are the feasible collect points $C^{\mu}_1, \hdots, C^{\mu}_{n_{\mu}}$ obtained in Step~2.1. All other nodes (the release points, $p^{\mu}_o$ and $p^{\mu}_f$) are treated as clusters with a single node and therefore always chosen. The graph is created as follows (lines~41-49):
\begin{itemize}
    \item Add nodes for $p^{\mu}_o$, $p^{\mu}_f$, all release points, and all feasible collect points for each tour.
    \item Connect $p^{\mu}_o$ to the first release point $X^{\mu}_{1,1}$ with an edge weighted by the UGV travel time between them.
    \item Connect the release point of each tour to each of its feasible collect points, with the edge weight being the maximum of the UAV's tour time and UGV travel time from the release to the collect point.
    \item Connect each feasible collect point to the release point of the next tour (or $p^{\mu}_f$ in the case of the last tour), with an edge weight equal to the maximum of the UGV's travel time and the UAV's recharge time after the previous tour.
\end{itemize}
With this structure, the resulting GTSP solution has the form $(p^{\mu}_o, X^{\mu}_{1,1}, c^{\mu}_1, \hdots, X^{\mu}_{n,1}, c^{\mu}_{n_{\mu}}, p^{\mu}_f)$, where $c^{\mu}_i$ is the selected collect point for the $i^{\text{th}}$ tour from $C^{\mu}_i$. In line~50, GTSP is solved using the GLNS heuristic \cite{smith2016GLNS} with a bounded runtime of $\sigma n_{\mu}^3$, where $\sigma$ is a user-defined constant and $n_{\mu} = |\mathcal{P}^{\mu}_{\text{UAV}}|$. Note that the start and end nodes are fixed. This procedure returns the optimized collect points $(c^{\mu}_1, \dots, c^{\mu}_{n_{\mu}})$. Finally, $X^{\mu}$ is completed by setting $X^{\mu}_{i,n_{\mu}+2}$ to $c^{\mu}_i$ for all $1 \leq i \leq n_{\mu}$ (line~52). After $X^{\mu}$ is completed, it is added to the overall mission plan $X$ (line~53).
\begin{algorithm}
\renewcommand{\AlCapSty}[1]{\normalfont\footnotesize{\textbf{#1}}\unskip}\footnotesize
\SetAlgoLined
\SetAlgoNoEnd
\DontPrintSemicolon

\SetKwInOut{Input}{Input}
\SetKwInOut{Output}{Output}

\Input{
    $\mathcal{P}_{\text{UAV}}$ (Set of $n$ points to visit),
    $m$ (Number of teams),
    $\{p_o^{\mu}\}_{\mu=1}^m$, $\{p_f^{\mu}\}_{\mu=1}^m$ (Team Start/Finish positions),
    $\ell$ (Distance function),
    $\mathcal{Q}_f^g$ (Feasible Ground Space),
    $\mathcal{Q}_f^a$ (Feasible Air Space),
    $\tau_a$, $\tau_g$ (UAV, UGV travel time functions),
    $\overline{\tau_{\text{a}}}$ (Maximum flight time),
    $\delta_a, \delta_g$ (Robustness parameters),
    $\sigma$ (GLNS time limit constant)
}
\Output{
    $X = \{X^{\mu}\}_{\mu=1}^{m}$ (Mission Plan)
}
\SetKw{continue}{continue}

\caption{RSPECT, Robust Multi UAV-UGV Planning Algorithm}
\label{alg:integrated_path_planning}

\textbf{Step 1: Assign points to teams and construct visit sequences}\;
\textbf{Initialize:} $\{\mathcal{S}^{\mu}\}_{\mu=1}^{m} \gets \emptyset$;\ \ 
$\{\mathcal{P}_{\text{UAV}}^{\mu}\}_{\mu=1}^{m} \gets \emptyset$\;
\uIf{$m = 1$}{
    $\mathcal{P}_{\text{UAV}}^{1} \gets \mathcal{P}_{\text{UAV}}$;\ \
    $\mathcal{S}^{1} \gets \mathrm{NN-TSP}(p_o^{1}, \mathcal{P}_{\text{UAV}}, p_f^{1})$~\cite{rosenkrantz1977analysis}\footnotemark[7]\;
}
\Else{
    $\mathcal{P}_{\text{rem}} \gets \mathcal{P}_{\text{UAV}}$\;
    \For{$\mu \gets 1$ \KwTo $m$}{
        $D_{\mu} \gets 0$\;
    }
    \While{$\mathcal{P}_{\text{rem}} \neq \emptyset$}{
        $p \gets$ $\mathrm{UniformSample}$\footnotemark[7] from $\mathcal{P}_{\text{rem}}$\;
        \For{$\mu \gets 1$ \KwTo $m$}{
            $\mathcal{T}_{\mu} \gets \mathrm{NN-TSP}(p_o^{\mu}, \mathcal{P}_{\text{UAV}}^{\mu} \cup \{p\}, p_f^{\mu})$\cite{rosenkrantz1977analysis}\footnotemark[6];\ \ 
            $D'_{\mu} \gets \mathrm{cost}(\mathcal{T}_{\mu})$\;
        }
        $\mu^{\star} \gets \mathop{\mathrm{arg\,min}}\limits_{\mu \in \{1,\ldots,m\}} \max\big(\{D_{\nu} \mid \nu \neq \mu\} \cup \{D'_{\mu}\}\big)$\;
        $\mathcal{P}_{\text{UAV}}^{\mu^{\star}} \gets \mathcal{P}_{\text{UAV}}^{\mu^{\star}} \cup \{p\}$;\ \
        $\mathcal{S}^{\mu^{\star}} \gets \mathcal{T}_{\mu^{\star}}$;\ \
        $D_{\mu^{\star}} \gets D'_{\mu^{\star}}$;\ \
        $\mathcal{P}_{\text{rem}} \gets \mathcal{P}_{\text{rem}} \setminus \{p\}$\;
    }
}
\textbf{Output of Step 1:} $\{\mathcal{S}^{\mu}\}_{\mu=1}^{m}$, $\{\mathcal{P}_{\text{UAV}}^{\mu}\}_{\mu=1}^{m}$\;

\textbf{Step 2: Generate feasible plans for each team}\;
\textbf{Initialize:} $X \leftarrow \varnothing$\;
\For{$\mu \leftarrow 1$ \KwTo $m$}{
$n_{\mu} \leftarrow |\mathcal{P}_{\text{UAV}}^{\mu}|$\;

\textbf{Step 2.1: Create feasible tours from the sequence $\mathcal{S}^{\mu}$}

\textbf{Initialize:} $X^{\mu} \leftarrow \emptyset$, $k \leftarrow 1$ \;
\For{$i \leftarrow 1$ \KwTo $n_{\mu}$}{
    \uIf{$k = n_{\mu}+1$}{
        $X_{i,1}^{\mu} \leftarrow X_{i-1,1}^{\mu}$
    }
    \Else{ $X_{i,1}^{\mu} \leftarrow \pi(\mathcal{S}_k^{\mu}) {=} \mathop{\mathrm{arg\,min}}\limits_{p \in \mathcal{Q}_f^g} \ell(p, \mathcal{S}_k^{\mu})$\;
    }
    $X_{i,n_{\mu}+2}^{\mu} \leftarrow X_{i,1}^{\mu}$

    \For{$j \leftarrow 2$ \KwTo $n_{\mu}+1$}{
        \uIf{\(k \leq n_{\mu}\)}{
                $X_{i,j}^{\mu} \leftarrow \mathcal{S}_k^{\mu}$;\ \
                \(C \leftarrow \{\}\)\;
            \For{$r \leftarrow 2$ \KwTo $j$}{
                \uIf{$\tau_{a}(X^{\mu}_{i}) {+} \delta_a {\leq} \overline{\tau_{a}}$ \textbf{and} \(\tau_g(X^{\mu}_{i,1}, \pi(X^{\mu}_{i,r})){+} \delta_g {\leq} \overline{\tau_{a}} \)}{
                    \(C \leftarrow C \cup \{\pi(X^{\mu}_{i,r})\}\)
                }
            }
            \uIf{\(C \neq \emptyset\)}{
                \(C^{\mu}_i \leftarrow C\);\ \
                $k \leftarrow k + 1$\;
            }\Else{
                $X^{\mu}_{i,j} \leftarrow X^{\mu}_{i,j-1}$\;
            }
        }\Else{
        $X^{\mu}_{i,j} \leftarrow X^{\mu}_{i,j-1}$\;
        }
}
}
\textbf{Step 2.2: Create a feasible plan \(X^{\mu}\) using GTSP}

\(V^{\mu} \leftarrow \bigcup_i (C^{\mu}_i \cup \{X_{i,1}^{\mu}\}) \cup \{p_o^{\mu}, p_f^{\mu}\}\);\ \
\(G^{\mu} \leftarrow \mathrm{Graph}(V^{\mu})\)\;
$G^{\mu}.\mathrm{add\_edge}(p_o^{\mu},X_{1,1}^{\mu}, \tau_g(p_o^{\mu}, X_{1,1}^{\mu}))$\;
\For{$i \leftarrow 1$ \KwTo $n_{\mu}$}{
\ForEach{$c \in C^{\mu}_i$}{
    $X_{i,n_{\mu}+2}^{\mu} \leftarrow c$;\ \
    $G^{\mu}.\mathrm{add\_edge}(X^{\mu}_{i,1},c,\max[ \tau_g(X^{\mu}_{i,1},c), \tau_a(X^{\mu}_{i})])$\;
    \uIf{$i=n_{\mu}$}{
        $G^{\mu}.\mathrm{add\_edge}(c,p^{\mu}_f,\max[\tau_g(c,p^{\mu}_f), \tau_c(X^{\mu}_{i})])$
    }\Else{
        $G^{\mu}.\mathrm{add\_edge}(c,X^{\mu}_{i{+}1,1},\max[ \tau_g(c,X^{\mu}_{i{+}1,1}), \tau_c(X^{\mu}_{i})])$\;
       }
}
}
$(c^{\mu}_1, \dots, c^{\mu}_{n_{\mu}}) \leftarrow \mathrm{GLNS}(G^{\mu},p^{\mu}_o,p^{\mu}_f, \{C^{\mu}_1, {\dots}, C^{\mu}_{n_{\mu}}\}, \sigma)$ \cite{smith2016GLNS}\footnotemark[7]\;
\For{$i \leftarrow 1$ \KwTo $n_{\mu}$}{
    \(X^{\mu}_{i,n_{\mu}+2} = c^{\mu}_i\)\;}
    $X \leftarrow X \cup \{X^{\mu}\}$
}

\Return{$X$}
\end{algorithm}
\footnotetext[7]{{Other TSP/GTSP solvers and \textcolor{black}{selection methods} can be used in these steps. However, this might change the worst-case time complexity of Alg.\ref{alg:integrated_path_planning} given in Theorem \ref{th-1}.}} \setcounter{footnote}{7}
\begin{theorem}
\label{th-1}
     Alg.\ref{alg:integrated_path_planning} returns a feasible solution to \eqref{opt_problem} if one exists, and it has the worst-case time complexity of $\mathcal{O}(mn+n^3)$, where $m$ is the number of UAV-UGV teams and $n=|\mathcal{P}_{\text{UAV}}|$ is the number of monitoring points.\footnote{In practice, the number of monitoring points ($n$) is typically much larger than the number of UAV-UGV teams ($m$). Hence, the worst-case complexity of Alg.\ref{alg:integrated_path_planning}, $\mathcal{O}(mn+n^3)$, is usually dominated by the growth of $n$ as $\mathcal{O}(n^3)$.}
\end{theorem}
\color{black}
\begin{proof} (\textit{Feasibility}:)  
In Step~1, Alg.\ref{alg:integrated_path_planning} partitions $\mathcal{P}_{\text{UAV}}$ among $m$ teams, guaranteeing each monitoring point is assigned to exactly one team. In other words, after Step~1 is completed, $\mathcal{P}_{\text{UAV}} = \bigcup_{\mu=1}^{m} \mathcal{P}^{\mu}_{\text{UAV}}$ and $\mathcal{P}^{\mu}_{\text{UAV}} \cap \mathcal{P}^{\nu}_{\text{UAV}} = \emptyset$ for $\mu \neq \nu$. When $m=1$ (line 4), all points are trivially assigned to the single team. Otherwise, the while loop (lines 9-14) iterates until $\mathcal{P}_{\text{rem}}$ is empty, and at each iteration, exactly one point is removed from $\mathcal{P}_{\text{rem}}$ and added to one team's set $\mathcal{P}^{\mu^{\star}}_{\text{UAV}}$ (line 14). Furthermore, in lines 4 and 12, Alg.\ref{alg:integrated_path_planning} computes a TSP that contains all the points in $\mathcal{P}^{\mu}_{\text{UAV}}$ as the sequence $\mathcal{S}^{\mu}$. In lines 23-27 of Step~2.1, \(X^{\mu}_{i,1}\) is either chosen as a point in \(\mathcal{Q}_f^g\) or as the release point of the previous tour. In either case, \(X^{\mu}_{i,1}\) satisfies (\ref{eq:con4}). \color{black}In Step~2.1, lines 28-39, each $\mathcal{S}^{\mu}_k$ is added to the current tour (row $i$) if there is a collect point for which the resulting tour would not violate (\ref{eq:con2}) and (\ref{eq:con3}). Note that if a feasible solution to \eqref{opt_problem} exists, then each $p\in \mathcal{P}_{\text{UAV}}$ must be reachable from the closest feasible point on the ground without violating the constraint \eqref{eq:con2}. Hence, under Alg.\ref{alg:integrated_path_planning}, for each run of the main \texttt{for} loop in Step~2.1, unless \(k > n_{\mu}\), at least one waypoint \(\mathcal{S}^{\mu}_k\) is added to the plan matrix of the team, \(X^{\mu}\) (as in line 30). Since $X^{\mu}$ has $n_{\mu}=|\mathcal{S^{\mu}}|$ rows, this means that all the waypoints in \(\mathcal{S^{\mu}}\) will be included in $X^{\mu}$ created by Alg.\ref{alg:integrated_path_planning}. Moreover, if there are no remaining points in $\mathcal{S^{\mu}}$, i.e., $k>n_{\mu}$, or adding $\mathcal{S}_k$ to tour $i$ would violate \eqref{eq:con2} or \eqref{eq:con3}, then  \(X^{\mu}_{i,j}\) is assigned to be the same as the previous point \(X^{\mu}_{i,j-1}\) (for \(j \leq n_{\mu}+1\), lines 37 and 39). Hence, \eqref{eq:con5} is also satisfied. In Step~2.2, collect points for each tour are chosen from a set of points which all satisfy \eqref{eq:con3} and \eqref{eq:con4}. Finally, the mission plan $X = \{X^{\mu}\}_{\mu=1}^{m}$ satisfies (\ref{eq:con1}) collectively since $\mathcal{P}_{\text{UAV}} = \bigcup_{\mu=1}^{m} \mathcal{P}^{\mu}_{\text{UAV}}$ and thus constitutes a feasible solution to (\ref{opt_problem}).

(\textit{Complexity}:) 
\color{black}
In Step~1, the while loop (line 9) runs $n$ times, since each iteration removes exactly one point from $\mathcal{P}_{\text{rem}}$. Each iteration performs three operations: (i) selecting a point $p$ from 
$\mathcal{P}_{\text{rem}}$ uniformly at random (line 10), which takes $\mathcal{O}(1)$ time; (ii) computing a candidate TSP path for each of the $m$ teams (lines 11-12); and (iii) determining $\mu^{\star}$ by finding the minimizer among $m$ teams (line 13) and updating the assigned team's data structures (line 14), which takes $\mathcal{O}(m)$ time. For the candidate TSP paths in (ii), we use a Nearest-Neighbor heuristic 
\cite{rosenkrantz1977analysis}, which has a worst-case complexity of $\mathcal{O}(n^2)$ for $n$ points. Accordingly, the computation of candidate TSP paths across $m$ teams takes $\sum_{\mu=1}^{m} \mathcal{O}({n_{\mu}}^2)$ time. Note that $\sum_{\mu=1}^{m} \mathcal{O}({n_{\mu}}^2) \leq \mathcal{O}(n^2)$, since the number of points allocated to teams satisfy $\sum_{\mu=1}^{m} n_{\mu} \leq n$ at any iteration. Hence, each iteration of the while loop costs $\mathcal{O}(n^2 + m)$, and the total runtime of Step~1 is $\mathcal{O}(n \cdot (n^2 + m)) = \mathcal{O}(mn + n^3)$ in the worst-case.\footnote{\color{black}Note that for $m=1$ the assignment is trivial (all points go to the single team) and Step~1 reduces to a single TSP on $n$ points (line 4), with $\mathcal{O}(n^2)$ complexity (using NN-TSP heuristic \cite{rosenkrantz1977analysis}).}

\color{black}
Steps 2.1-2.2 are executed once per team. Time complexity of projection function \(\pi(.)\) (lines 26 and 32) depends on the definition of obstacles that are known a priori. Clearly, the obstacles that are unknown a priori has no impact on this offline computation. 
Assuming that the known obstacles are expressed as convex polytopes, finding the closest point on the boundary of such polytope to an interior point can be formulated as a convex quadratic program for each face, which is solvable in polynomial time in the number of faces and the dimension of the space \cite{QP}. Checking if a point is inside a convex polytope reduces to a linear programming feasibility problem and is therefore polynomial-time solvable. Therefore, if there are \(b\) such obstacles, the total cost of a projection scales as \(\mathcal{O}(b)\), assuming each obstacle has a bounded number of faces. So the total time to find the ground projections of all $n_{\mu}$ monitoring points assigned to team $\mu$ is \(\mathcal{O}(n_{\mu}b)\), which is dominated by the runtime of Step~2.1 (discussed next), \(\mathcal{O}({n_{\mu}}^3)\), assuming that the number of known obstacles $b$ is bounded.

In Step~2.1 of Alg.\ref{alg:integrated_path_planning}, all \texttt{for} loops iterate at most \(n_{\mu}\) times. While this step includes projection operations, their costs have already been accounted for in the overall projection time discussed earlier. All remaining operations within these loops take constant time. Consequently, the time complexity of Step~2.1, excluding projections, is \(\mathcal{O}({n_{\mu}}^3)\).

In Step~2.2 of Alg.\ref{alg:integrated_path_planning}, we use GLNS to solve the GTSP. Although GLNS does not have worst-case polynomial-time complexity, we have restricted its runtime to \(\mathcal{O}({n_{\mu}}^3)\) by enforcing a time limit of \(\sigma {n_{\mu}}^3\), where \(\sigma\) is a user-defined constant. 

Steps 2.1-2.2 are executed once per team. Since the overall runtime of these steps for each team $\mu$ is \(\mathcal{O}({n_{\mu}}^3)\), the total runtime of Step~2 will be $\sum_{\mu=1}^m \mathcal{O}({n_\mu}^3)$. Since $\{\mathcal P_{\mathrm{UAV}}^{\mu}\}_{\mu=1}^m$ is a partition of $\mathcal P_{\mathrm{UAV}}$, we have $\sum_{\mu=1}^m n_\mu = n$ and $n_\mu \geq 0$. For any set of non-negative numbers, the sum of their cubes is less than or equal to cube of their sum. This gives the inequality $\sum_{\mu=1}^m {n_\mu}^3 \leq (\sum_{\mu=1}^m n_\mu)^3$, which simplifies to $\sum_{\mu=1}^m {n_\mu}^3 \leq n^3$. Hence the worst-case complexity of Step~2 is $\mathcal{O}(n^3)$. Combining this with the complexity of Step~1, $\mathcal{O}(mn + n^3)$, we obtain the overall complexity of Alg.\ref{alg:integrated_path_planning} as $\mathcal{O}(mn+n^3)$.
\end{proof}

\color{black}

\subsection{Robustness Analysis}
\label{subsec:robustness}
Among the various sources of uncertainty, a priori unknown obstacles are arguably one of the most common and significant ones. Such unknown obstacles may render certain release/collect points in $X = \{X^{\mu}\}_{\mu=1}^{m}$ inaccessible to UGVs during actual execution and/or increase traverse time by the UAVs/UGVs during each tour. However, if \( X \) has sufficient tour robustness values, i.e., $\hat{\delta}^{\mu,i}_a(X)$ and $ \hat{\delta}^{\mu,i}_g(X)$ are sufficiently large, it may still be possible to follow the plan by only making small modifications in the release/collect points (as opposed to a complete replaning). In this section, we formally analyze the robustness of our method to such unknown obstacles.

Let \( \tilde{\mathcal{Q}}_f \subseteq \mathcal{Q}_f \) be the actual feasible space, i.e., the feasible points remaining after the a priori unknown obstacles are removed from $\mathcal{Q}_f$. Accordingly, let \( \tilde{\mathcal{Q}}_f^g \subseteq \mathcal{Q}_f^g \) be the set of actual feasible points on ground. Same as before, we assume that both \( \tilde{\mathcal{Q}}_f \) and  \( \tilde{\mathcal{Q}}_f^g \) are connected sets. Moreover, $\mathcal{P}_{\text{UAV}} \subseteq \tilde{\mathcal{Q}}_f$. These mild assumptions imply that the mission is still achievable within the actual feasible space. Furthermore, let
$\tilde{\tau}_a:\tilde{\mathcal{Q}}_f^{\,n+2} \rightarrow \mathbb{R}_{\ge 0}$
and 
$\tilde{\tau}_g : \tilde{\mathcal{Q}}_f \times \tilde{\mathcal{Q}}_f \rightarrow \mathbb{R}_{\ge 0}$ be actual traversal-time functions i.e., the times required for the UAVs and UGVs to traverse the corresponding paths within the actual feasible space $\tilde{\mathcal{Q}}_f$.
Finally, let $\tilde{X} = \{\tilde{X}^{\mu}\}_{\mu=1}^{m}$ be any modified overall plan.
In the following theorem, we state sufficient conditions under which the plan generated by Alg.~\ref{alg:integrated_path_planning}, $X = \{X^{\mu}\}_{\mu=1}^{m}$, can be executed in the face of unknown obstacles by simply adjusting the release and collect points as needed.

\begin{theorem}\label{th2}
Let $X = \{X^{\mu}\}_{\mu=1}^{m}$ be a feasible solution to~\eqref{opt_problem} with tour robustness values
$\hat{\delta}^{\mu,i}_a(X),\hat{\delta}^{\mu,i}_g(X)$ (see Def.\ref{def1}), and
let $\tilde{X} = \{\tilde{X}^{\mu}\}_{\mu=1}^{m}$ be a modified overall plan such that 
\begin{itemize}
    \item Each $\tilde{X}^{\mu}$ differs from $X^{\mu}$ only in its release and collect points, i.e., $\tilde{X}^{\mu}_{i,1}, \tilde{X}^{\mu}_{i,n_{\mu}+2} \in \tilde{\mathcal{Q}}^{g}_f$ and ${X}^{\mu}_{i,1}, {X}^{\mu}_{i,n_{\mu}+2} \in {\mathcal{Q}}^{g}_f$.
    \item {The following holds for each $\tilde{X}^{\mu}$ and $\forall i=\{1,\dots,n_{\mu}\}$:
    \begin{align}
    \tilde{\tau}_a(\tilde X^{\mu}_{i}) &\le \tau_a(X^{\mu}_{i}) + \hat{\delta}^{\mu,i}_a(X), \label{eq:condA}\\
    \tilde{\tau}_g(\tilde X^{\mu}_{i,1}, \tilde X^{\mu}_{i,n_{\mu}+2}) &\le \tau_g(X^{\mu}_{i,1}, X^{\mu}_{i,n_{\mu}+2}) + \hat{\delta}^{\mu,i}_g(X), 
    \label{eq:condG}
    \end{align}}
\end{itemize}
where $n_{\mu}=|\mathcal{P}^{\mu}_{\text{}UAV}|$. Then under $\tilde X$, every point in $\mathcal P_{\text{UAV}}$ is visited while respecting the energy limitation of the UAVs, i.e.,
$\tilde{\tau}_a(\tilde{X}^{\mu}_{i}) \le \overline{\tau_{a}}$ and
$\tilde{\tau}_g(\tilde{X}^{\mu}_{i,1}, \tilde X^{\mu}_{i,n_{\mu}+2}) \le \overline{\tau_{a}}$ for all $i=\{{1,{\dots},n_{\mu}\}}$ and $\mu=\{1,{\dots},m\}$.
\end{theorem}

\begin{proof}
As a feasible solution to \eqref{opt_problem}, $X$ satisfies \eqref{eq:con1}, and UAVs visit all points in $\mathcal{P}_{\text{UAV}}$ under $X$. Since \( X \) and \( \tilde{X} \) only differ in release and collect points, the UAVs visit all points in $\mathcal{P}_{\text{UAV}}$ also under \( \tilde{X} \). By Def. \ref{def1}, for all $i$ and $\mu$, we have:
\begin{align}
   \tau_a(X^\mu_i) + \hat{\delta}^{\mu,i}_a(X) \le \overline{\tau_{a}} \label{eq:def1condA}, \\
   \tau_g(X^\mu_{i,1}, X^\mu_{i,n_\mu+2}) + \hat{\delta}^{\mu,i}_g(X) \le \overline{\tau_{a}}. \label{eq:def1condG}
\end{align}
Using (\ref{eq:condA}) and (\ref{eq:def1condA}) gives
$\tilde{\tau}_a(\tilde{X}^{\mu}_{i}) \le \overline{\tau_{a}}$.
Similarly, using (\ref{eq:condG}) and (\ref{eq:def1condG}) gives
$\tilde{\tau}_g(\tilde{X}^{\mu}_{i,1}, \tilde{X}^{\mu}_{i,n_{\mu}+2}) \le \overline{\tau_{a}}$.
Consequently, under \( \tilde{X} \), all points in $\mathcal{P}_{\text{UAV}}$ are visited while respecting the energy limits of the UAVs.
\end{proof}
\color{black}


\subsection{Special Case: VTOL UAV, Unknown Ground Obstacles}
\label{sec:vtol}
In light of Theorem \ref{th2}, when obstacles on the ground are the primary source of uncertainty, an offline overall plan  $X = \{X^{\mu}\}_{\mu=1}^{m}$ with sufficiently robust tours allows the UAV-UGV teams to complete the mission by only adjusting their own release/collect points if needed, without any other changes to $X = \{X^{\mu}\}_{\mu=1}^{m}$. Under some mild conditions on the vehicles and the environment, as given in Assumption \ref{assump1}, this online modification of the release/collect points becomes even simpler: feasible points within a certain radius, which depends on $\hat{\delta}^{\mu,i}_a(X)$, from the original release/collect points can be used as the modified release/collect points as long as the resulting travel time of UGV between those release and collect points are not increased by more than $\hat{\delta}^{\mu,i}_g(X)$ relative to $X$.
\begin{assumption}\footnote{\textcolor{black}{Note that Assumption \ref{assump1} is not required for our generic results in Theorems \ref{th-1} and \ref{th2}. It is only utilized in Corollary \ref{cor}, which enables a simple, local plan modification strategy for the UAV-UGV teams in response to ground obstacles (Remark \ref{remark}).} The VTOL capability is typical for multi-rotor UAVs, whose take-off and landing maneuvers are usually repeatable within a constant duration $\tau_{TL}$. Obstacle-free VTOL can be achieved when the obstacles are represented via their bounding boxes. The increase in flight time due to obstacles above the minimum flight altitude ${\underbar{$z$}}$ is typically negligible when ${\underbar{$z$}}$ is sufficiently large. }
UAVs can perform vertical take-off/landing (VTOL) at any feasible ground point to/from a minimum flight altitude ${\underbar{$z$}}$ in constant time $\tau_{TL}$. No obstacles prevent VTOL or exist at altitudes higher than  ${\underbar{$z$}}$.
\label{assump1}
\end{assumption}

\begin{corollary}
\label{cor}
    Let Assumption \ref{assump1} hold, let $X = \{X^{\mu}\}_{\mu=1}^{m}$ be a feasible solution to \eqref{opt_problem} with robustness \( \hat\delta_a^{\mu,i}(X), \hat\delta^{\mu,i}_g(X)\), and let $\tilde{X} = \{\tilde{X}^{\mu}\}_{\mu=1}^{m}$ be any modified plan that differs from \( X \) only in some release/collect points while satisfying the following conditions for all $\mu=\{1,\dots,m\}$ and $i=\{{1},\dots,n_{\mu}\}$:
\begin{equation}
    \|\tilde{X}^\mu_{i,1}-X^\mu_{i,1}\| + \|\tilde{X}^\mu_{i,n_\mu+2}-X^\mu_{i,n_\mu+2}\| \le v_H^{\text{avg}} \cdot \hat{\delta}^{\mu,i}_a(X),
    \label{eq:prop1conditionA} 
\end{equation}
\begin{equation}
    \tilde{\ell}_{g,i}^\mu \le \ell_{g,i}^\mu + v_g^{\text{avg}} \cdot \hat{\delta}^{\mu,i}_g(X),
    \label{eq:prop1conditionG}
\end{equation}
where  $v^{\text{avg}}_{H}$ is the average speed the UAV can sustain along any fixed-altitude path, $v_g^{\text{avg}}$ is the average speed that can be sustained by the UGV, $\ell^{\mu}_{g,i}$ and $\tilde{\ell}^{\mu}_{g,i}$ denote the lengths of shortest feasible ground paths between release and collect points of $\mu^{\text{th}}$ team's $i^{\text{th}}$ tour under $X$ and $\tilde{X}$. Then under $\tilde X$, every point in $\mathcal P_{\text{UAV}}$ is visited while respecting the energy limitation of the UAVs, i.e.,
$\tilde{\tau}_a(\tilde{X}^{\mu}_{i}) \le \overline{\tau_{a}}$ and
$\tilde{\tau}_g(\tilde{X}^{\mu}_{i,1}, \tilde X^{\mu}_{i,n_{\mu}+2}) \le \overline{\tau_{a}}$ for all $\mu$ and $i$. 
\end{corollary}
\begin{proof}
We will show that, under Assumption \ref{assump1}, (\ref{eq:prop1conditionA}) and (\ref{eq:prop1conditionG}) imply (\ref{eq:condA}) and (\ref{eq:condG}) given in the premise of Theorem \ref{th2}.

 For any \( X \) and \(\tilde{X}\) satisfying the premise, let's consider each non-trivial tour (rows with at least one point in $\mathcal{P}_{\text{UAV}}$) under the VTOL behavior in Assumption \ref{assump1}. Without loss of generality, let's assume that these non-trivial tours $i$ contain no repetitions of the release point, i.e., $X^{\mu}_{i,j}=\tilde{X}^{\mu}_{i,j} \in \mathcal{P}_{\text{UAV}}$ for all $j \in \{2, {\dots}, n_{\mu}+1\}$. Furthermore, since no obstacles prevent vertical take-off/landing or exist above the minimum flight altitude under Assumption \ref{assump1}, we can express the times for such non-trivial tours as:
\begin{equation*}
\resizebox{\columnwidth}{!}{$
\begin{aligned}
   {\tau}_a({X}^{\mu}_i) &= 2{\tau_{TL}} + \Delta(\mathcal{X}^\mu_{i,1},X^\mu_{i,2}) + \Delta(X^\mu_{i,2},X^\mu_{i,3}) + \cdots + \Delta(X^\mu_{i,n_\mu+1},\mathcal{X}^\mu_{i,n_\mu+2}), \\
   {\tilde{\tau}}_a(\tilde{X}^{\mu}_i) &= 2{\tau_{TL}} + \Delta(\tilde{\mathcal{X}}^\mu_{i,1},X^\mu_{i,2}) + \Delta(X^\mu_{i,2},X^\mu_{i,3}) + \cdots + \Delta(X^\mu_{i,n_\mu+1},\tilde{\mathcal{X}}^\mu_{i,n_\mu+2}),
\end{aligned}
$}
\end{equation*}
where the function $\Delta$ gives the minimum time it takes a UAV to travel between the respective points \textcolor{black}{(note that it is not a separate function but component of $\tau_a(\cdot), \tilde{\tau}_a(\cdot)$ under VTOL behavior from Assumption \ref{assump1})}, and  $\mathcal{X}^\mu_{i,1}, \tilde{\mathcal{X}}^\mu_{i,1}$ and $\mathcal{X}^\mu_{i,n_\mu+2}, \tilde{\mathcal{X}}^\mu_{i,n_\mu+2}$ are the additional waypoints due to VTOL. For example, $\mathcal{X}^\mu_{i,1}$ is at the minimum flight altitude ${\underbar{$z$}}$ with the same $x$ and $y$ coordinates as the release point $X^\mu_{i,1}$. In both equations above, total time for take-off and landing takes constant time $2\tau_{TL}$. Additionally, in both cases, the time to traverse $X^{\mu}_{i,2}, \hdots, X^{\mu}_{i,n_{\mu}+1}$ are equal (no obstacles above ${\underbar{$z$}}$). Accordingly,
\begin{equation}
\resizebox{\columnwidth}{!}{$
    \tilde{\tau}_a(\tilde{X}^{\mu}_i) {-} {\tau}_a({X}^{\mu}_i) {=} \Delta(\tilde{\mathcal{X}}^\mu_{i,1}, X^\mu_{i,2}){+} \Delta(X^\mu_{i,n_\mu+1}, \tilde{\mathcal{X}}^\mu_{i,n_\mu+2}){-}\Delta(\mathcal{X}^\mu_{i,1},X^\mu_{i,2}){-}\Delta(X^\mu_{i,n_\mu+1}, \mathcal{X}^\mu_{i,n_\mu+2}).
    \label{eq:partial_time_differences}
    $}
\end{equation}

Moreover, since $\Delta$ denotes the minimum travel time between two points, it satisfies the triangle inequality, i.e., for any three points $p_1,p_2,p_3$, $\Delta(p_1,p2) \leq\Delta(p_1,p_3)+\Delta(p_3,p_2)$. Hence,
\begin{align}
  \Delta(\tilde{\mathcal{X}}^\mu_{i,1}, X^\mu_{i,2}) \le \Delta(\tilde{\mathcal{X}}^\mu_{i,1}, \mathcal{X}^\mu_{i,1}) + \Delta(\mathcal{X}^\mu_{i,1}, X^\mu_{i,2}) \label{eq:triangle_ineq_1},\\
  \Delta(X^\mu_{i,n_\mu+1},\tilde{\mathcal{X}}^\mu_{i,n_\mu+2}) \le \Delta(\tilde{\mathcal{X}}^\mu_{i,n_\mu+2}, \mathcal{X}^\mu_{i,n_\mu+2}) + \Delta(X^\mu_{i,n_\mu+1}, \mathcal{X}^\mu_{i,n_\mu+2}) .\label{eq:triangle_ineq_2}
\end{align}

Using \eqref{eq:partial_time_differences}, \eqref{eq:triangle_ineq_1},  and \eqref{eq:triangle_ineq_2}, we obtain:
\begin{equation}
   \tilde{\tau}_a(\tilde{X}^{\mu}_i) - {\tau}_a({X}^{\mu}_i) \le \Delta( \tilde{\mathcal{X}}^\mu_{i,1}, \mathcal{X}^\mu_{i,1}) + \Delta( \tilde{\mathcal{X}}^\mu_{i,n_\mu+2}, \mathcal{X}^\mu_{i,n_\mu+2}).
   \label{eq:triangle_distance}
\end{equation}

Since the distance between the respective release/collect points under $X$ and $\tilde{X}$ are equal to the distance between the projections of those points at the minimum flight altitude, i.e., ${\| \tilde{X}_{i,1} - X_{i,1} \| = \| \tilde{\mathcal{X}}_{i,1} - \mathcal{X}_{i,1} \|}$ and ${\| \tilde{X}_{i,n+2} - X_{i,n+2} \| = \| \tilde{\mathcal{X}}_{i,n+2} - \mathcal{X}_{i,n+2} \| }$, \eqref{eq:triangle_distance} implies
\begin{equation}
\tilde{\tau}_a(\tilde{X}_{i}) - {\tau}_a({X}_{i}) \le \frac{\|\tilde{X}^\mu_{i,1}-X^\mu_{i,1}\|}{v^{\text{avg}}_H} + \frac{\|\tilde{X}^\mu_{i,n_\mu+2}-X^\mu_{i,n_\mu+2}\|}{v^{\text{avg}}_H} ,\label{eq:air_dist_bound}
\end{equation} where $v^{\text{avg}}_H$ is the the average speed the UAV can sustain along any fixed-altitude path. Using (\ref{eq:prop1conditionA}) and (\ref{eq:air_dist_bound}), we obtain
\begin{equation}
    \tilde\tau_a(\tilde{X}^\mu_i) - \tau_a(X^\mu_i) \le  \hat{\delta}^{\mu,i}_a(X).
    \label{eq:final_uav_bound}
\end{equation}
Hence, $\tilde{X}$ satisfies (\ref{eq:condA}) of Theorem \ref{th2}. Furthermore, for the UGVs, using $v^{\text{avg}}_g$, we can write:
\begin{equation}
    \tilde\tau_g(\tilde{X}^\mu_{i,1}, \tilde{X}^\mu_{i,n_{\mu+2}}) - \tau_g(X^\mu_{i,1}, {X}^\mu_{i,n_{\mu+2}})  \le \frac{\tilde{\ell}_{g,i}^\mu - \ell_{g,i}^\mu}{v_g^{\text{avg}}}.
    \label{eq:ugv_dist_to_timee}
\end{equation}
Using (\ref{eq:prop1conditionG}) and (\ref{eq:ugv_dist_to_timee}) together, we have
\begin{equation}
    \tilde\tau_g(\tilde{X}^\mu_{i,1}, \tilde{X}^\mu_{i,n_{\mu+2}}) - \tau_g(X^\mu_{i,1}, {X}^\mu_{i,n_{\mu+2}})  \le \hat{\delta}^{\mu,i}_g(X).
    \label{eq:final_ugv_bound}
\end{equation}
Accordingly, $\tilde{X}$ also satisfies (\ref{eq:condG}) of Theorem \ref{th2}. Hence, using Theorem \ref{th2}, we can conclude that every point in $\mathcal{P}_{\text{UAV}}$ is visited under $\tilde{X}$ while respecting the energy limitations of the UAVs.
\end{proof}


\begin{remark}
\label{remark}
Corollary \ref{cor} shows that, under Assumption \ref{assump1}, UAV-UGV teams can follow the offline plan $X$ by only changing the release/collect points if they are occupied as long as 
\begin{itemize}
    \item The combined deviation of modified release and collect points from the offline plan remain within $v_H^{\text{avg}} \cdot \hat{\delta}^{\mu,i}_a(X)$.
    \item The ground path length between the modified release and collect points increases by no more than $v_g^{\text{avg}} \cdot \hat{\delta}^{\mu,i}_g(X)$.
\end{itemize}
 Figure \ref{fig:robustness} provides a visualization of these two conditions.\footnote{\color{black}\eqref{eq:prop1conditionA} and \eqref{eq:prop1conditionG} provide a selection rule for $(\delta_a, \delta_g)$. Let $S$ denote the maximum tolerable per-tour detour; then $\delta_a \geq S/v_H^{\text{avg}}$ and $\delta_g \geq S/v_g^{\text{avg}}$ suffice. In practice, $S$ can be estimated from pre-deployment surveys, terrain or wind data, or vehicle test data. When only obstacle statistics are available, results from continuous first-passage percolation on Poisson Boolean obstacle models~\cite{goueretheret2017positivity, willot2015power} characterize how expected shortest-path lengths scale with obstacle density and can serve as a guide for estimating $S$.}
\end{remark}

\begin{figure}[htpb]
    \centering
    \includegraphics[width=1.0\columnwidth]{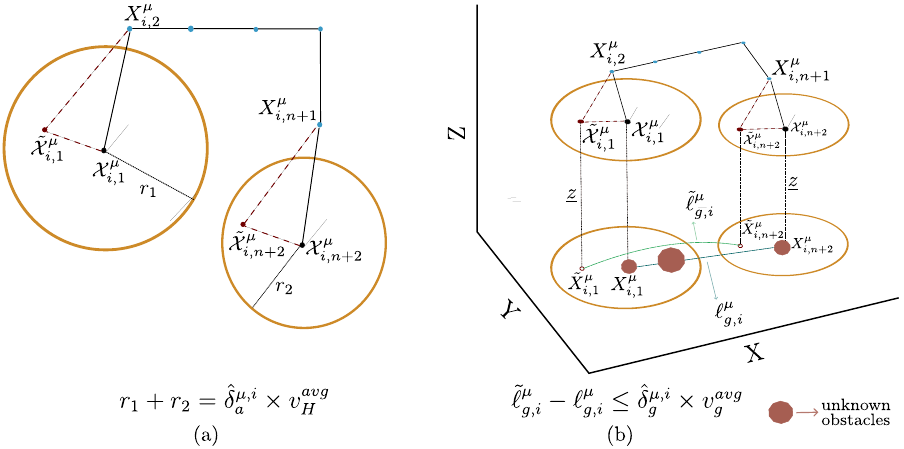}
    \caption{Illustration for Remark \ref{remark}. A tour with its original and modified release/collect points due to unknown obstacles. (a) top view, (b) 3D view.} 
    \label{fig:robustness}
\end{figure}

\color{black}
\section{Simulations}
\label{sec:simulations}

We evaluate the performance, scalability, and component contributions of Alg.~\ref{alg:integrated_path_planning} through sets of simulations: (1) a scalability analysis across different $m$ and $n$ (Section~\ref{subsec:scalability}); (2) comparisons against Branch and 
Cut~\cite{mitchell2002branch}, two metaheuristics --Simulated Annealing (SA) and Genetic Algorithm (GA)-- and three state-of-the-art heuristics from the literature~\cite{ropero2019terra,yu2018algorithms,
maini2019coverage} developed for similar energy-aware UAV-UGV planning (Sections~\ref{sec:bnc}--\ref{sec:sota}); and (3) an ablation study evaluating the contribution of each algorithmic component, the sensitivity 
to robustness parameters $\delta_a$ and $\delta_g$, and the algorithm's performance under randomly generated initial and final team positions (Section~\ref{sec:ablation}). All implementations were performed using Python 3.9 on an Intel i9-13900H CPU with 16GB RAM. In all simulations, a VTOL UAV is considered with a minimum flight altitude of $\underline{z} = 100\,m$ (see Assumption~\ref{assump1}). In practice, most UAVs move horizontally at speeds greater than their vertical speeds~\cite{guo2020vertical}, 
which is consistent with our configuration (Table~\ref{tab:mission_parameters}).

For different values of $m$ (number of UAV-UGV teams) and $n$ (number of monitoring points), 25 random $\mathcal{P}_{\text{UAV}}$ were created via uniform distribution from $\mathcal{Q}_f^a$ with a fixed $z$-coordinate equal to $\underline{z}$ inside a cuboid environment as defined in~\eqref{env}, with $\bar{x} = \bar{y} = 4000\,m$. Unless otherwise specified, computation time and overall mission time $\max_{\mu}\tau(X^{\mu})$ are reported as mean $\pm$ standard deviation over those 25 realizations. Mission parameters are presented in Table~\ref{tab:mission_parameters} and the following linear recharge 
relationship is used:
\begin{equation}
    \tau_{\text{c}}(X_{i}^{\mu}) = \gamma \times \max\!\big(\tau_a({X}^{\mu}_{i}), 
    \tau_g({X}^{\mu}_{i,1}, {X}^{\mu}_{i,n_{\mu}+2})\big), \quad \gamma \geq 0,
    \label{eq:linear_charging_equation}
\end{equation}
where $\gamma$ is the recharge ratio. Hence, $\gamma=0$ approximates a 
battery swap system, whereas a larger $\gamma$ models a slower recharging 
mechanism.

Across all simulations throughout this section, $\tau_a$ and $\tau_g$ are instantiated using Euclidean 
path lengths and the fixed speeds in Table~\ref{tab:mission_parameters}. For the UAV, horizontal and vertical speeds are treated separately to model VTOL behavior. The UGV travel time is computed from Euclidean ground distances at the fixed ground speed. Simulations are conducted in an obstacle-free environment to enable fair comparison with~\cite{ropero2019terra,yu2018algorithms,maini2019coverage} (see Sec. \ref{sec:sota} for details).

\begin{table}[!htpb]
\centering
\caption{System Parameters Used in the Simulations}
\label{tab:mission_parameters}
\begin{tabular}{|c|c|c|c|}
\hline
\textbf{Parameter} & \textbf{Value} & \textbf{Parameter} & \textbf{Value} \\
\hline
UGV Speed & $2.5\,m/s$ & UAV Speed & H: $10\,m/s$, V: $2\,m/s$ \\
\hline
$\overline{\tau_{\text{a}}}$ (Max flight time) & $600\,s$ & 
$\gamma$ (Recharge Ratio) & 1 \\
\hline
\end{tabular}
\end{table}
\color{black}

\color{black}
\subsection{Scalability of the Proposed Algorithm}
\label{subsec:scalability}

The performance of Alg.~\ref{alg:integrated_path_planning} was evaluated for 
$n \in \{25, 50, 75, 100\}$ and $m \in \{1, 2, 3, 4, 7, 10\}$. For all 
teams, $p_{\text{o}}^{\mu} = (0,0,0)$ and $p_{\text{f}}^{\mu} = (4000,4000,0)$ 
for all $\mu \in \{1,\ldots,m\}$, and mission parameters from 
Table~\ref{tab:mission_parameters} are used throughout. Results are 
summarized in Table~\ref{tab:heuristic_performance}.


Computation time increases with $n$ as expected from the
$\mathcal{O}(mn + n^3)$ complexity established in
Theorem~\ref{th-1}, with empirical scaling exponents consistently
well within the theoretical cubic bound across all values of $m$.
The exponent decreases for larger $m$, consistent with the structure
of Alg.~\ref{alg:integrated_path_planning}: as $m$ increases, each
team is assigned fewer waypoints on average, reducing the cost of
the per-team TSP subproblems that dominate the $mn+n^3$ term in the
worst-case bound.

Overall mission time $\max_{\mu}\tau(X^{\mu})$ generally increases with $n$, 
as larger problems require longer tours, and decreases with $m$, as teams 
operate in parallel and share the workload. The results confirm that 
Alg.~\ref{alg:integrated_path_planning} scales well across a wide range of 
problem sizes and team configurations, completing in under $7\,\text{s}$ 
even at $n=100$, $m=2$ i.e., the most computationally demanding configuration 
tested.


\begin{table}[htpb]
\centering
\color{black}
\small
\setlength{\tabcolsep}{4pt}
\caption{Computation Time and Overall Mission Time, $\max_{\mu}\tau(X^{\mu})$, under Alg.~\ref{alg:integrated_path_planning} for Different Numbers of Teams ($m$) and Problem Sizes ($n$).}
\label{tab:heuristic_performance}
\begin{tabular}{ccc cc cc}
\toprule
 & & \multicolumn{2}{c}{$\boldsymbol{m=1}$} & \multicolumn{2}{c}{$\boldsymbol{m=2}$} \\
\cmidrule(lr){3-4}\cmidrule(lr){5-6}
 & $n$ & \textbf{Comp. (s)} & $\boldsymbol{\max_{\mu}\tau(X^{\mu})}$ \textbf{(s)}
       & \textbf{Comp. (s)} & $\boldsymbol{\max_{\mu}\tau(X^{\mu})}$ \textbf{(s)} \\
\midrule
& 25  & $0.08 \pm 0.01$ & $5150 \pm 630$ & $0.26 \pm 0.03$ & $4130 \pm 580$ \\
& 50  & $0.22 \pm 0.03$ & $6250 \pm 680$ & $1.00 \pm 0.07$ & $4070 \pm 390$ \\
& 75  & $0.50 \pm 0.10$ & $7270 \pm 390$ & $2.79 \pm 0.18$ & $4630 \pm 390$ \\
& 100 & $0.82 \pm 0.17$ & $8200 \pm 540$ & $6.30 \pm 0.44$ & $5160 \pm 410$ \\
\midrule
 & & \multicolumn{2}{c}{$\boldsymbol{m=3}$} & \multicolumn{2}{c}{$\boldsymbol{m=4}$} \\
\cmidrule(lr){3-4}\cmidrule(lr){5-6}
 & $n$ & \textbf{Comp. (s)} & $\boldsymbol{\max_{\mu}\tau(X^{\mu})}$ \textbf{(s)}
       & \textbf{Comp. (s)} & $\boldsymbol{\max_{\mu}\tau(X^{\mu})}$ \textbf{(s)} \\
\midrule
& 25  & $0.29 \pm 0.03$ & $3650 \pm 560$ & $0.36 \pm 0.04$ & $3500 \pm 480$ \\
& 50  & $0.86 \pm 0.04$ & $3710 \pm 340$ & $0.86 \pm 0.04$ & $3530 \pm 430$ \\
& 75  & $2.16 \pm 0.10$ & $3870 \pm 380$ & $1.92 \pm 0.08$ & $3690 \pm 360$ \\
& 100 & $4.57 \pm 0.32$ & $4100 \pm 480$ & $3.79 \pm 0.28$ & $3830 \pm 470$ \\
\midrule
 & & \multicolumn{2}{c}{$\boldsymbol{m=7}$} & \multicolumn{2}{c}{$\boldsymbol{m=10}$} \\
\cmidrule(lr){3-4}\cmidrule(lr){5-6}
 & $n$ & \textbf{Comp. (s)} & $\boldsymbol{\max_{\mu}\tau(X^{\mu})}$ \textbf{(s)}
       & \textbf{Comp. (s)} & $\boldsymbol{\max_{\mu}\tau(X^{\mu})}$ \textbf{(s)} \\
\midrule
& 25  & $0.50 \pm 0.03$ & $3420 \pm 210$ & $0.57 \pm 0.06$ & $3430 \pm 200$ \\
& 50  & $1.07 \pm 0.05$ & $3510 \pm 270$ & $1.27 \pm 0.04$ & $3480 \pm 220$ \\
& 75  & $1.94 \pm 0.06$ & $3580 \pm 390$ & $2.20 \pm 0.06$ & $3530 \pm 240$ \\
& 100 & $3.25 \pm 0.16$ & $3580 \pm 440$ & $3.41 \pm 0.11$ & $3610 \pm 290$ \\
\bottomrule
\end{tabular}
\end{table}

\color{black}

\color{black}
\subsection{Comparison to Branch and Cut}
\label{sec:bnc}

Table~\ref{tab:global_vs_proposed} presents a comparison between Branch and Cut and Alg.~\ref{alg:integrated_path_planning} for $m \in \{1, 2\}$ and $n \in \{2,3,4,5\}$. Branch and Cut is implemented using Gurobi~\cite{gurobi} with a \texttt{MipGap} tolerance of $5\%$ and a \texttt{TimeLimit} of $200\,\text{s}$. Encoding details and further parameters are available in Appendix \ref{app:bnc}. For each instance, Branch and Cut returns either a solution guaranteed to be within $5\%$ of the global optimum or, if the time limit is reached, the best feasible solution found so far. Start and end positions are set to $p_{\text{o}}^{\mu} = (0,0,0)$ and $p_{\text{f}}^{\mu} = (4000,4000,0)$ 
for all $\mu \in \{1,\ldots,m\}$.

The achieved optimality gap of Branch and Cut is consistently small: mean gaps remain below $2\%$ across all configurations. The median gap is $0\%$ for all but two configurations ($(m\!=\!1, n\!=\!5)$ and $(m\!=\!2, n\!=\!4)$ with medians of $0.52\%$ and $1.28\%$ respectively), indicating that Branch and Cut recovers the global optimum in the majority of instances. The solutions returned by Branch and Cut therefore serve as a meaningful near-optimal reference for evaluating Alg.~\ref{alg:integrated_path_planning}'s solution quality. However, Branch and Cut computation time grows rapidly with both $n$ and $m$, consistent with the NP-hardness of~\eqref{opt_problem}: computation 
times reach $10.5 \pm 8.8\,\text{s}$ and $16 \pm 12\,\text{s}$ for $(m\!=\!1, n\!=\!5)$ and $(m\!=\!2, n\!=\!5)$ respectively.

In contrast, Alg.~\ref{alg:integrated_path_planning} completes in under $0.10\,\text{s}$ across all tested configurations. The mission times produced by Alg.~\ref{alg:integrated_path_planning} are competitive with the near-optimal Branch and Cut reference: the gap narrows as $n$ increases for both $m=1$ and $m=2$, suggesting that solution quality improves relative to the near-optimal bound at larger $n$.

\begin{table}[htpb]
\centering
\color{black}
\small
\caption{Comparing Computation Time and Overall Mission Time, \textbf{$\max_{\mu}\tau(X^{\mu})$ (s)}, Between Branch and Cut and Alg.~\ref{alg:integrated_path_planning} for Different Problem Sizes, $n$ and Different Number of Teams, $m$.}
\label{tab:global_vs_proposed}
\begin{tabular}{cc ccc cc}
\toprule
 & & \multicolumn{3}{c}{\textbf{Branch and Cut}} & \multicolumn{2}{c}{\textbf{Alg. 1}} \\
\cmidrule(lr){3-5}\cmidrule(lr){6-7}
$m$ & $n$ & \textbf{Comp. (s)} & \textbf{$\max_{\mu}\tau(X^{\mu})$ (s)} & \textbf{Gap (\%)} & \textbf{Comp. (s)} & \textbf{$\max_{\mu}\tau(X^{\mu})$ (s)} \\
\midrule
\multirow{4}{*}{1}
 & 2 & $0.07 \pm 0.10$ & $2380 \pm 230$ & $1.45 \pm 1.94$ & $0.02 \pm 0.01$ & $2860 \pm 620$ \\
 & 3 & $0.22 \pm 0.24$ & $2590 \pm 370$ & $1.00 \pm 1.65$ & $0.03 \pm 0.00$ & $2910 \pm 580$ \\
 & 4 & $0.98 \pm 1.51$ & $3120 \pm 620$ & $0.74 \pm 1.43$ & $0.03 \pm 0.01$ & $3250 \pm 670$ \\
 & 5 & $10.5  \pm 8.8$  & $3030 \pm 620$ & $1.79 \pm 2.01$ & $0.03 \pm 0.00$ & $3320 \pm 650$ \\
\midrule
\multirow{4}{*}{2}
 & 2 & $0.21 \pm 0.23$ & $2290 \pm 71$  & $0.35 \pm 0.94$ & $0.06 \pm 0.01$ & $2950 \pm 180$ \\
 & 3 & $1.17 \pm 0.94$ & $2480 \pm 280$ & $1.30 \pm 1.66$ & $0.07 \pm 0.01$ & $3000 \pm 290$ \\
 & 4 & $3.83 \pm 3.46$ & $2920 \pm 660$ & $1.76 \pm 1.67$ & $0.08 \pm 0.01$ & $3210 \pm 410$ \\
 & 5 & $16.4 \pm 11.9$ & $3030 \pm 530$ & $1.84 \pm 2.17$ & $0.09 \pm 0.01$ & $3170 \pm 240$ \\
\bottomrule
\end{tabular}
\end{table}
\color{black}

\color{black}
\subsection{Comparison to Metaheuristics}

We compare two metaheuristics, GA and SA, against Alg.~\ref{alg:integrated_path_planning} on 5 randomly generated $\mathcal{P}_{\text{UAV}}$ instances (uniform distribution from $\mathcal{Q}^a_f$) with $n=10$, evaluated in the setting $m \in \{1, 2\}$, $p_{\text{o}}^{\mu} = (0,0,0)$ and $p_{\text{f}}^{\mu} = (4000,4000,0)$ for all $\mu \in \{1,\ldots,m\}$. Larger $n$ values are not computationally tractable for GA and SA, which motivates this reduced problem size. Each instance was repeated 3 times per $\mathcal{P}_{\text{UAV}}$ (15 total runs), and results are reported as mean $\pm$ standard deviation (practically min-max range) across runs. In contrast, 
Alg.~\ref{alg:integrated_path_planning} is deterministic and yields the same result for the same input. For this setting, Alg.~\ref{alg:integrated_path_planning} achieves $\max_{\mu}\tau(X^{\mu}) = 4480 \pm 680\,\text{s}$ with a computation time of $0.03 \pm 0.00\,\text{s}$ for $m=1$, and $\max_{\mu}\tau(X^{\mu}) = 3720 \pm 840\,\text{s}$ with a computation time of $0.13 \pm 0.01\,\text{s}$ for $m=2$. 

SA and GA hyperparameters were tuned systematically using Optuna~\cite{akiba2019optuna}, a black-box hyperparameter optimization framework, prior to evaluation, as detailed in the following subsections. Both algorithms run to a fixed termination condition: a maximum number of cooling steps for SA and a 
maximum number of generations for GA, with the resulting computation times reported alongside overall mission times. Even at the largest tested configurations, neither SA nor GA approaches the overall mission time of Alg.~\ref{alg:integrated_path_planning}, despite requiring $100$-$200\, \text{s}$ compared to under $0.2\,\text{s}$.
\color{black}

\color{black}
\subsubsection{Comparison to Simulated Annealing}
\label{sec:sa}

The performance of SA is sensitive to hyperparameter selection, namely the initial temperature $T_{\text{max}}$, the final temperature $T_{\text{min}}$, the number of cooling steps, and the cooling schedule~\cite{kirkpatrick1983optimization,nourani1998sa}. Hyperparameters were tuned using Optuna~\cite{akiba2019optuna}. The search space covered $T_{\text{max}} \in \{500,750, 1000, 1500, 2000\}$, $T_{\text{min}} \in \{0.01, 0.1, 1\}$, steps $\in \{50\text{k}, 100\text{k}, 150\text{k}, 200\text{k}\}$, and cooling schedules $\in \{\text{logarithmic, linear, quadratic}\}$. Selection was based on minimizing $\max_{\mu}\tau(X^{\mu})$ while maintaining low variance across runs, yielding $T_{\text{min}} = 0.1$ and a logarithmic cooling schedule, which provides a principled balance between exploration and exploitation~\cite{nourani1998sa}. The remaining parameters $T_{\text{max}} \in \{500, 750, 1000\}$ and steps $\in \{100\text{k}, 150\text{k}, 200\text{k}\}$ were then varied to assess sensitivity, with results reported in Table~\ref{tab:sa_results}.

The initial feasible solution provided to SA is constructed via a round-robin assignment, where points are distributed one per tour in sequence. For $m=2$, points are split evenly between teams prior 
to the round-robin tour assignment. This yields $\max_{\mu}\tau(X^{\mu}) 
= 10540 \pm 1040\,\text{s}$ for $m=1$ and $\max_{\mu}\tau(X^{\mu}) = 
12350 \pm 1320\,\text{s}$ for $m=2$.

We note that even at the largest configuration (200k steps), SA requires approximately $108\,\text{s}$
for $m=1$ and $203\,\text{s}$ for $m=2$, several orders of magnitude greater than Alg.~\ref{alg:integrated_path_planning}, which completes in under $0.15\,\text{s}$ in all tested configurations. Despite this substantial computation budget, SA consistently yields higher $\max_{\mu}\tau(X^{\mu})$ across all settings. Furthermore, improvements tend to plateau beyond 100k
steps in most configurations, suggesting that additional computation causes diminishing returns.

The results also reveal that SA degrades with increasing $m$. For $m=2$, both overall mission times and their standard deviations increase substantially relative to $m=1$ under identical hyperparameter settings. For example, at $T_{\text{max}} = 1000$ and 100k steps, $\max_{\mu}\tau(X^{\mu})$ increases
from $5500 \pm 480\,\text{s}$ for $m=1$ to $8000 \pm 1200\,\text{s}$ for $m=2$, with both the mean and variance deteriorating significantly. This indicates that SA struggles to navigate the larger joint
assignment-and-routing search space that arises with multiple teams. This degradation is expected for a generic metaheuristic, as the neighborhood structure explored by SA does not adapt to the combinatorial structure introduced by multi-team coordination. In contrast, Alg.~\ref{alg:integrated_path_planning} explicitly exploits the structure of~(\ref{opt_problem}), achieving $\max_{\mu}\tau(X^{\mu}) = 4480 \pm 680\,\text{s}$ for $m=1$ and $3720 \pm 840\,\text{s}$ for $m=2$.
\color{black}

\begin{table}[htpb]
\centering
\small
\color{black}
\setlength{\tabcolsep}{2pt}
\caption{Simulated Annealing Average Computation Time and Overall Mission Time, ${\max_{\mu}\tau(X^{\mu})}$, for Different $T_{\text{max}}$, Number of Steps, and Number of Teams $m$}
\label{tab:sa_results}
\begin{tabular}{cc ccc ccc}
\toprule
 & & \multicolumn{3}{c}{\textbf{Comp. Time (s)}}
   & \multicolumn{3}{c}{$\boldsymbol{\max_{\mu}\tau(X^{\mu})}$ \textbf{(s)}} \\
\cmidrule(lr){3-5}\cmidrule(lr){6-8}
 & & \multicolumn{3}{c}{\textbf{Steps}}
   & \multicolumn{3}{c}{\textbf{Steps}} \\
\cmidrule(lr){3-5}\cmidrule(lr){6-8}
$m$ & $T_{\text{max}}$
  & \textbf{100k} & \textbf{150k} & \textbf{200k}
  & \textbf{100k} & \textbf{150k} & \textbf{200k} \\
\midrule
\multirow{3}{*}{1}
& 500  & $53.76 \pm 0.27$ & $80.77 \pm 0.24$ & $107.75 \pm 0.66$
       & $5840 \pm 510$   & $5500 \pm 830$   & $5230 \pm 690$ \\
& 750  & $53.85 \pm 0.21$ & $80.76 \pm 0.36$ & $107.55 \pm 0.41$
       & $5730 \pm 540$   & $5280 \pm 490$   & $4960 \pm 500$ \\
& 1000 & $54.02 \pm 0.42$ & $80.83 \pm 0.39$ & $107.75 \pm 0.56$
       & $5500 \pm 480$   & $5370 \pm 730$   & $4850 \pm 500$ \\
\midrule
\multirow{3}{*}{2}
& 500  & $101.0 \pm 0.38$ & $151.7 \pm 0.68$ & $202.5 \pm 1.5$
       & $6950 \pm 910$   & $6200 \pm 1100$  & $5700 \pm 1300$ \\
& 750  & $101.2 \pm 0.50$ & $152.1 \pm 0.74$ & $202.5 \pm 0.96$
       & $7600 \pm 1400$  & $5900 \pm 1300$  & $4900 \pm 1300$ \\
& 1000 & $101.4 \pm 0.44$ & $152.7 \pm 1.5$  & $202.3 \pm 0.73$
       & $8000 \pm 1200$  & $6600 \pm 1200$  & $5500 \pm 1700$ \\
\bottomrule
\end{tabular}
\end{table}

\color{black}

\color{black}
\subsubsection{Comparison to Genetic Algorithm}
\label{sec:ga}

The performance of GA is sensitive to hyperparameter selection, including population size, number of generations, mutation rate, number of mating parents, and selection strategy~\cite{gen1999genetic}. Hyperparameters were tuned using Optuna~\cite{akiba2019optuna}, with the search space covering population size $\in \{100, 200, 300, 400, 500\}$, generations $\in \{100, 150, 200, 250\}$, mutation rate $\in \{5, 10, 25, 50\}\%$, number of mating parents $\in \{50, 75, 100\}$, and selection strategies $\in \{\text{steady-state, roulette wheel, tournament}\}$. Selection was based on minimizing $\max_{\mu}\tau(X^{\mu})$ while maintaining low variance across runs, yielding a mutation rate of $50\%$, 75 mating parents, and steady-state selection (SSS), which replaces a portion of the population at each generation to promote gradual convergence~\cite{smith2007steady}. The remaining parameters, 
population size $\in \{300, 400, 500\}$ and generations $\in \{100, 150, 200\}$, were then varied to assess sensitivity, with results reported in Table~\ref{tab:ga_results}.

The initial feasible solution provided to GA is constructed via the same round-robin assignment as SA.

Even at the largest configuration (population size of 500 and 200 generations), GA requires approximately 
$62\,\text{s}$ for $m=1$ and $123\,\text{s}$ for $m=2$, several orders of magnitude greater than Alg.~\ref{alg:integrated_path_planning}. Despite this computation budget, GA consistently yields higher $\max_{\mu}\tau(X^{\mu})$ across all settings, with the best achieved value of $7000 \pm 860\,\text{s}$ for $m=1$ and $9700 \pm 1500\,\text{s}$ for $m=2$, compared to $4960 \pm 500\,\text{s}$ and $4900 \pm 1300\,\text{s}$ for SA. As expected, larger populations and more generations generally improve overall mission time, though improvements plateau with additional computation.

The results further reveal that GA degrades with increasing $m$, similar to SA. For $m=2$, overall mission times increase by approximately $35$-$45\%$ relative to $m=1$ under identical settings, and standard deviations nearly double. As with SA, this degradation is expected for a
generic metaheuristic, as it does not exploit the combinatorial structure of (\ref{opt_problem}). In contrast, Alg.~\ref{alg:integrated_path_planning} explicitly leverages this 
structure, achieving consistently lower $\max_{\mu}\tau(X^{\mu})$ with negligible computation time.

\begin{table}[htpb]
\centering
\small
\color{black}
\setlength{\tabcolsep}{2pt}
\caption{Genetic Algorithm Average Computation Time and Overall Mission Time, $\max_{\mu}\tau(X^{\mu})$ for Different Population Sizes, Generations, and Number of Teams $m$}
\label{tab:ga_results}
\begin{tabular}{cc ccc ccc}
\toprule
 & & \multicolumn{3}{c}{\textbf{Comp. Time (s)}}
   & \multicolumn{3}{c}{$\boldsymbol{\max_{\mu}\tau(X^{\mu})}$ \textbf{(s)}} \\
\cmidrule(lr){3-5}\cmidrule(lr){6-8}
 & & \multicolumn{3}{c}{\textbf{Generations}}
   & \multicolumn{3}{c}{\textbf{Generations}} \\
\cmidrule(lr){3-5}\cmidrule(lr){6-8}
$m$ & \textbf{Pop.}
  & \textbf{100} & \textbf{150} & \textbf{200}
  & \textbf{100} & \textbf{150} & \textbf{200} \\
\midrule
\multirow{3}{*}{1}
& 300 & $19.12 \pm 0.47$ & $28.51 \pm 0.71$ & $38.35 \pm 0.82$
      & $8000 \pm 1200$  & $7600 \pm 1000$  & $7300 \pm 920$ \\
& 400 & $25.48 \pm 0.50$ & $38.13 \pm 0.51$ & $50.21 \pm 0.67$
      & $7700 \pm 1100$  & $7300 \pm 870$   & $7000 \pm 830$ \\
& 500 & $31.46 \pm 0.59$ & $46.90 \pm 0.57$ & $62.42 \pm 0.78$
      & $7500 \pm 1100$  & $7200 \pm 920$   & $7000 \pm 860$ \\
\midrule
\multirow{3}{*}{2}
& 300 & $37.57 \pm 0.92$ & $56.06 \pm 0.96$ & $74.6 \pm 1.4$
      & $10900 \pm 1700$ & $10500 \pm 1700$ & $10200 \pm 1600$ \\
& 400 & $49.28 \pm 0.67$ & $73.5 \pm 1.5$   & $97.8 \pm 1.2$
      & $10700 \pm 1600$ & $10300 \pm 1500$ & $10000 \pm 1500$ \\
& 500 & $61.6 \pm 1.1$   & $91.5 \pm 1.5$   & $122.8 \pm 1.6$
      & $10500 \pm 1800$ & $10100 \pm 1700$ & $9700 \pm 1500$ \\
\bottomrule
\end{tabular}
\end{table}

\color{black}

\subsection{Comparison to the State-of-the-Art}
\label{sec:sota}
We compared our approach to three heuristics \cite{ropero2019terra,yu2018algorithms,maini2019coverage} that investigate similar energy-aware UAV-UGV planning problems. A fundamental difference between our method and these works is that \cite{ropero2019terra,yu2018algorithms,maini2019coverage} explicitly assume 
a single UAV-UGV team and cannot generalize to multi-team scenarios ($m > 1$), whereas Alg.~\ref{alg:integrated_path_planning} accommodates multiple teams; these baselines are therefore evaluated only under the single-team configuration $m=1$. Another significant difference is the 
notion of robustness: in \cite{ropero2019terra,yu2018algorithms,maini2019coverage}, the map is assumed to be known a priori, making those methods prone to infeasibility in the presence of uncertainty and potentially requiring complete mission re-planning. In contrast, Alg.~\ref{alg:integrated_path_planning} 
maintains feasibility without replanning (Theorem~\ref{th2}).

Noting these fundamental differences, we focus on comparing the algorithms based on mission efficiency ($\max_{\mu}\tau(X^{\mu})$ for $m=1$) and computation time. To ensure a fair comparison, all methods are evaluated under identical conditions: an obstacle-free environment (consistent with the assumptions of \cite{ropero2019terra,yu2018algorithms,maini2019coverage}), the VTOL and recharge model in (\ref{eq:linear_charging_equation}) with $\gamma \in \{0, 1, 2\}$, mission parameters from Table~\ref{tab:mission_parameters}, and $p_{\text{o}}^{1} = p_{\text{f}}^{1} = (0,0,0)$ since \cite{ropero2019terra,yu2018algorithms,maini2019coverage} cannot accommodate $p_{\text{o}} \neq p_{\text{f}}$. TSPs and GTSPs of the baselines are solved using Nearest Neighbor heuristic for TSP \cite{rosenkrantz1977analysis} and GLNS \cite{smith2016GLNS}, respectively.

\color{black}
The three baseline methods are deterministic and run to completion without a user-defined termination criterion or computational budget: their computation times are therefore an intrinsic property of each algorithm and are reported as measured. All results are averaged over 25 randomly generated $\mathcal{P}_{\text{UAV}}$ instances (uniform distribution from $\mathcal{Q}_f^a$) for $n \in \{25, 50, 75, 100\}$, and reported as mean $\pm$ std.
\color{black}

\subsubsection{Comparison to \cite{ropero2019terra}} In
\cite{ropero2019terra}, the authors propose an algorithm, TERRA, to plan for a UAV-UGV team in a monitoring mission using charging stops (rendezvous points). TERRA has five stages: (1) Voronoi tessellations assign monitoring points ($\mathcal{P}_{\text{UAV}}$) to charging stops; (2) a greedy hitting set algorithm refines these assignments; (3) charging stop locations are refined with a Gravitational Optimization Algorithm; (4) a TSP is solved for the UGV path; and (5) a modified A* algorithm computes the path of energy-constrained UAV. Main limitations of the TERRA are: the UGV must remain stationary at charging stops while the UAV visits monitoring points (i.e., UGV cannot move until the tour/cycle is completed), which can severely slow missions when monitoring points are far apart. Moreover, there is no takeoff/landing modeling and no explicit charging model for UAV, as instantaneous are swaps assumed, which might be impractical in real-world deployment due to the need to carry high number of batteries.

Fig.\ref{fig:terra_comparison_mission_time} shows that Alg.\ref{alg:integrated_path_planning} consistently outperforms TERRA in terms of $\tau(X^{1})$ across all $\gamma$. As $n$ increases, the gap becomes more significant, even with $\gamma=0$, (as assumed in TERRA). 

\begin{figure}[htpb]
    \centering
    \includegraphics[width=\linewidth]{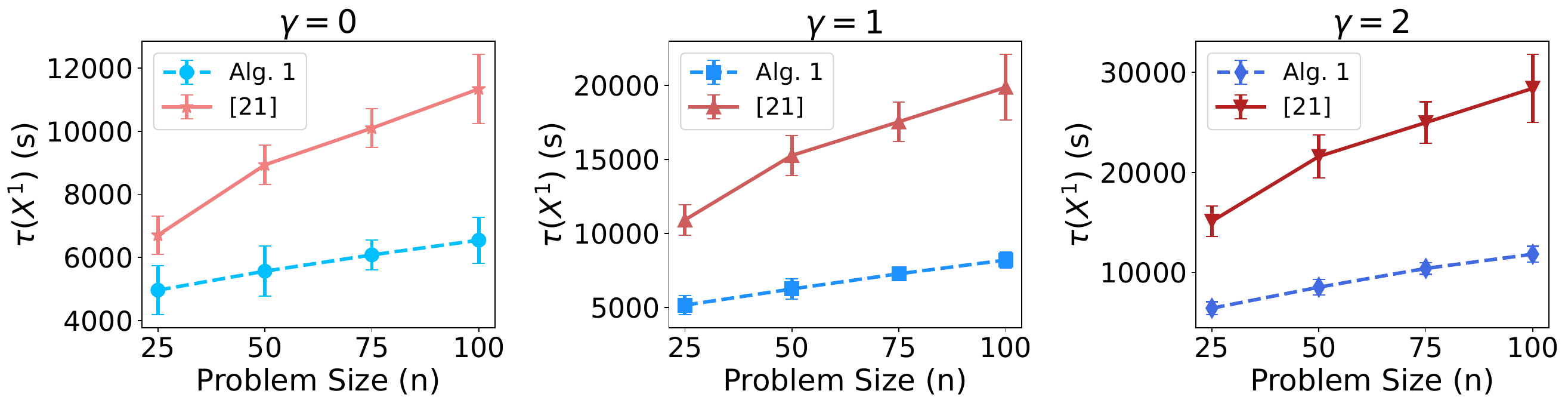}
    \caption{Comparing overall mission time ($\tau(X^{1})$) of Alg.\ref{alg:integrated_path_planning} and TERRA \cite{ropero2019terra} for different problem sizes ($n$) and recharge ratio ($\gamma$).}
    \label{fig:terra_comparison_mission_time}
\end{figure}

As Fig. \ref{fig:terra_comparison_computation_time} shows, Alg.~\ref{alg:integrated_path_planning} consistently achieves lower runtimes, the difference becomes more significant as $n$ increases, emphasizing Alg. \ref{alg:integrated_path_planning}'s scalability.

\begin{figure}[htpb]
    \centering
    \includegraphics[width=\linewidth]{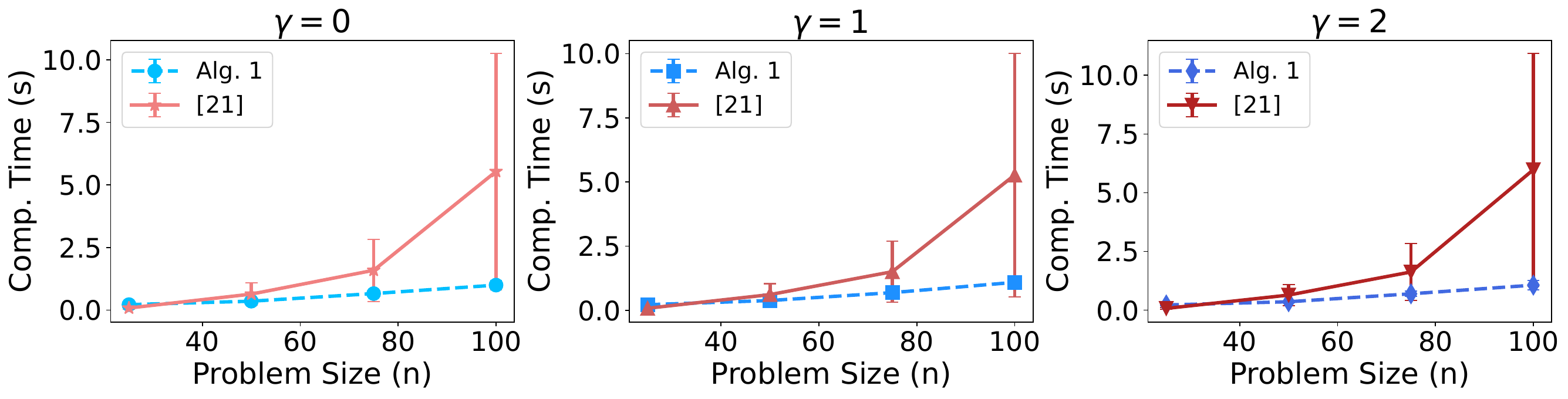}
    \caption{Comparing computation time of Alg.\ref{alg:integrated_path_planning} and TERRA \cite{ropero2019terra} for different problem sizes ($n$) and recharge ratio ($\gamma$).}
    \label{fig:terra_comparison_computation_time}
\end{figure}

\subsubsection{Comparison to \cite{maini2019coverage}}
In \cite{maini2019coverage}, the authors propose a heuristic for UAV-UGV route planning with energy constraints, where UGV operates in a road network. Main limitation of \cite{maini2019coverage} is the road network, restricting UGV  to predefined paths (e.g., square grid) and requires UAV monitoring points to be sufficiently close to this roads. Additionally, the method does not model charging. It first solves a TSP to generate the UAV path, and if the energy constraints are violated, a repair algorithm modifies the path (inserts charging stops from a set of predefined charging sites).
Alg.\ref{alg:integrated_path_planning} can operate under mobility constraints, including being constrained to a pre-defined road network. In our comparison, the UGV was constrained to a square road network (an example case in \cite{maini2019coverage}) with edges of $2.5\,km$, discretized at $10\,m$ resolution when running Alg.\ref{alg:integrated_path_planning}. Fig.\ref{fig:rn_comparison_mission_time} shows that even under this restriction (i.e., the use of road network, which naturally aligns with the setup of \cite{maini2019coverage}), Alg.\ref{alg:integrated_path_planning} achieves significantly lower $\tau(X^{1})$ when  $\gamma = 1,2$. Moreover, Fig. \ref{fig:rn_comparison_computation_time} shows that Alg.\ref{alg:integrated_path_planning} maintains competitive computational performance.

\begin{figure}[htpb]
    \centering
    \includegraphics[width=\linewidth]{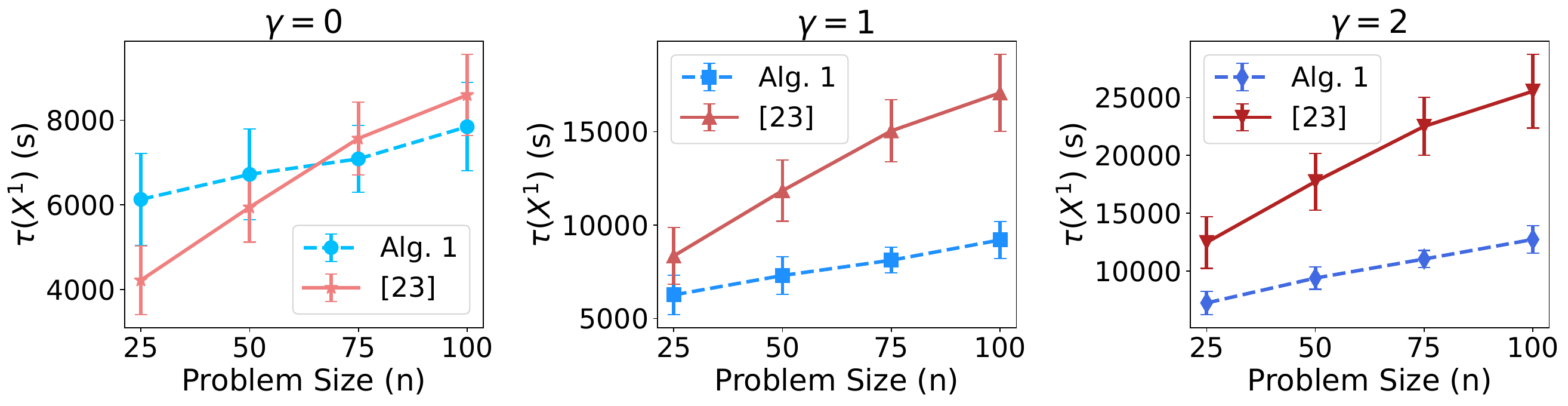}
    \caption{Comparing  overall mission time $(\tau(X^{1}))$ of Alg.\ref{alg:integrated_path_planning} and the method in \cite{maini2019coverage} for different problem sizes ($n$) and recharge ratios ($\gamma$).}
    \label{fig:rn_comparison_mission_time}
\end{figure}

\begin{figure}[htpb]
    \centering
    \includegraphics[width=\linewidth]{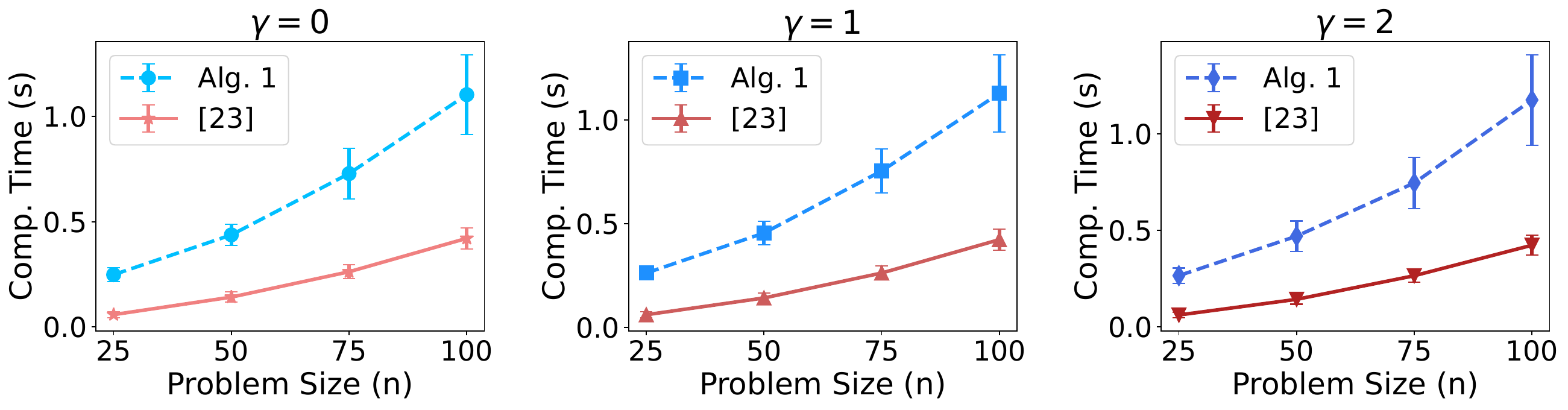}
    \caption{Comparing computation time of Alg.\ref{alg:integrated_path_planning} and the method in \cite{maini2019coverage} for different problem sizes ($n$) and recharge ratios ($\gamma$).}
    \label{fig:rn_comparison_computation_time}
\end{figure}

\subsubsection{Comparison to \cite{yu2018algorithms}}
In \cite{yu2018algorithms}, the authors solve a GTSP where nodes represent UAV monitoring points ($\mathcal{P}_{\text{UAV}}$) at discrete battery levels, and edges model transitions. They define two edge types for direct UAV flights with battery depletion, and for UAV recharging while transported by the UGV. Main limitation of this approach is the battery discretization, which effectively limits the consecutive monitoring points the UAV can visit in each tour. We used 10 battery levels for their algorithm to mitigate this limitation while keeping the problem tractable. 
In \cite{yu2018algorithms}, the UGV is expected to rendezvous with the UAV at a possibly different location at the end of each flight. Accordingly, the UAV needs to hover at the rendezvous location if it arrives earlier than the UGV. This behavior might cause energy constraint violation, i.e., battery depletion while hovering if the UGV does not arrive at the rendezvous location on time. Our approach avoids this problem with constraints (\ref{eq:con2}) and (\ref{eq:con3}), even in the face of uncertainty.

For evaluation, we tested various UGV speeds ($10$, $5$, and $2.5\, m/s$), problem sizes $(n)$, and recharge ratios ($\gamma$). Results are shown in Table \ref{tab:comparison_combined} (only averages of succesful runs are reported). Based on these, \cite{yu2018algorithms} shows high energy constraint violation with low UGV speeds, failing  up to 100\% of the missions at 2.5 $m/s$ for some instances. In terms of $\tau(X^{1})$, Alg.\ref{alg:integrated_path_planning} generally matches the performance of the method in \cite{yu2018algorithms} and outperforms it in some cases. This is largely because \cite{yu2018algorithms} uses discrete battery levels, and forces the UAV to takeoff/land frequently, causing additional delays. Moreover, Alg.\ref{alg:integrated_path_planning} has significantly lower computation times, as \cite{yu2018algorithms} suffers combinatorial growth caused by battery levels.

\begin{sidewaystable}
\centering
\caption{Evaluating Computation Time, Energy Constraint Violation, and Overall Mission Time, $\tau(X^{1})$, with different UGV speeds, problem sizes ($n$), and recharge ratios ($\gamma$). \textbf{Top:} Alg.~\ref{alg:integrated_path_planning}. \textbf{Bottom:} Method in \cite{yu2018algorithms}.}
\label{tab:comparison_combined}

\vspace{0.5cm}

\begin{tabular}{|c|c|ccc|ccc|ccc|}
\hline
\multirow{2}{*}{UGV Speed} & \multirow{2}{*}{$n$} & \multicolumn{3}{c|}{\textbf{Comp. Time (s)}} & \multicolumn{3}{c|}{\textbf{Energy Const. Violation}} & \multicolumn{3}{c|}{$\boldsymbol{\tau(X^{1})}$ (s)} \\
& & $\gamma{=}0$ & $\gamma{=}1$ & $\gamma{=}2$ & $\gamma{=}0$ & $\gamma{=}1$ & $\gamma{=}2$ & $\gamma{=}0$ & $\gamma{=}1$ & $\gamma{=}2$ \\
\hline
\multirow{4}{*}{2.5} & 25 & $0.09 \pm 0.02$ & $0.09 \pm 0.01$ & $0.09 \pm 0.01$ & 0\% & 0\% & 0\% & $5800 \pm 1100$ & $5150 \pm 630$ & $7300 \pm 1000$ \\
 & 50 & $0.24 \pm 0.05$ & $0.24 \pm 0.04$ & $0.24 \pm 0.05$ & 0\% & 0\% & 0\% & $6260 \pm 930$ & $6250 \pm 680$ & $9210 \pm 790$ \\
 & 75 & $0.51 \pm 0.12$ & $0.50 \pm 0.11$ & $0.50 \pm 0.12$ & 0\% & 0\% & 0\% & $6630 \pm 710$ & $7270 \pm 390$ & $10890 \pm 570$ \\
 & 100 & $0.92 \pm 0.20$ & $0.92 \pm 0.20$ & $0.91 \pm 0.20$ & 0\% & 0\% & 0\% & $6970 \pm 560$ & $8200 \pm 540$ & $12160 \pm 640$ \\
\hline
\multirow{4}{*}{5.0} & 25 & $0.08 \pm 0.01$ & $0.08 \pm 0.01$ & $0.08 \pm 0.01$ & 0\% & 0\% & 0\% & $2870 \pm 310$ & $3810 \pm 350$ & $5420 \pm 510$ \\
 & 50 & $0.23 \pm 0.05$ & $0.23 \pm 0.05$ & $0.23 \pm 0.05$ & 0\% & 0\% & 0\% & $3550 \pm 280$ & $5250 \pm 320$ & $7620 \pm 540$ \\
 & 75 & $0.49 \pm 0.13$ & $0.49 \pm 0.12$ & $0.48 \pm 0.12$ & 0\% & 0\% & 0\% & $4030 \pm 200$ & $6360 \pm 350$ & $9310 \pm 560$ \\
 & 100 & $0.87 \pm 0.20$ & $0.87 \pm 0.20$ & $0.86 \pm 0.19$ & 0\% & 0\% & 0\% & $4420 \pm 180$ & $7300 \pm 280$ & $10720 \pm 410$ \\
\hline
\multirow{4}{*}{10.0} & 25 & $0.07 \pm 0.01$ & $0.07 \pm 0.01$ & $0.07 \pm 0.01$ & 0\% & 0\% & 0\% & $2280 \pm 190$ & $3570 \pm 350$ & $5130 \pm 540$ \\
 & 50 & $0.22 \pm 0.05$ & $0.23 \pm 0.05$ & $0.22 \pm 0.05$ & 0\% & 0\% & 0\% & $2990 \pm 100$ & $5020 \pm 150$ & $7290 \pm 200$ \\
 & 75 & $0.48 \pm 0.12$ & $0.48 \pm 0.12$ & $0.48 \pm 0.12$ & 0\% & 0\% & 0\% & $3540 \pm 110$ & $6150 \pm 120$ & $9020 \pm 150$ \\
 & 100 & $0.86 \pm 0.19$ & $0.87 \pm 0.20$ & $0.86 \pm 0.19$ & 0\% & 0\% & 0\% & $4050 \pm 150$ & $7240 \pm 230$ & $10680 \pm 330$ \\
\hline
\end{tabular}

\vspace{0.8cm}

\begin{tabular}{|c|c|ccc|ccc|ccc|}
\hline
\multirow{2}{*}{UGV Speed} & \multirow{2}{*}{$n$} &
\multicolumn{3}{c|}{\textbf{Comp. Time (s)}} &
\multicolumn{3}{c|}{\textbf{Energy Const. Violation}} &
\multicolumn{3}{c|}{$\boldsymbol{\tau(X^{1})}$ (s)} \\
& & $\gamma{=}0$ & $\gamma{=}1$ & $\gamma{=}2$
  & $\gamma{=}0$ & $\gamma{=}1$ & $\gamma{=}2$
  & $\gamma{=}0$ & $\gamma{=}1$ & $\gamma{=}2$ \\
\hline
\multirow{4}{*}{2.5}
& 25  & - & - & - & 100\% & 100\% & 100\% & - & - & - \\
& 50  & - & - & $2.76 \pm 0.16$ & 100\% & 100\% & 64\% & - & - & $9470 \pm 390$ \\
& 75  & - & $7.45 \pm 0.73$ & $7.56 \pm 0.84$ & 100\% & 88\% & 40\% & - & $9160 \pm 370$ & $13040 \pm 250$ \\
& 100 & - & $14.4 \pm 1.4$ & $17.6 \pm 2.3$ & 100\% & 72\% & 16\% & - & $11770 \pm 330$ & $16130 \pm 400$ \\
\hline
\multirow{4}{*}{5.0}
& 25  & $0.52 \pm 0.05$ & $0.51 \pm 0.05$ & $0.50 \pm 0.08$ & 80\% & 8\% & 0\% & $2710 \pm 230$ & $3790 \pm 250$ & $4440 \pm 290$ \\
& 50  & $2.69 \pm 0.38$ & $2.48 \pm 0.30$ & $2.93 \pm 0.39$ & 44\% & 0\% & 0\% & $3630 \pm 190$ & $5470 \pm 190$ & $7130 \pm 160$ \\
& 75  & $6.84 \pm 0.82$ & $6.89 \pm 0.59$ & $9.90 \pm 2.0$ & 28\% & 0\% & 0\% & $4250 \pm 140$ & $7490 \pm 170$ & $9940 \pm 120$ \\
& 100 & $14.5 \pm 2.2$ & $15.4 \pm 2.4$ & $18.2 \pm 1.5$ & 4\% & 0\% & 0\% & $5000 \pm 200$ & $9450 \pm 200$ & $12770 \pm 150$ \\
\hline
\multirow{4}{*}{10.0}
& 25  & $0.48 \pm 0.04$ & $0.48 \pm 0.04$ & $0.54 \pm 0.06$ & 0\% & 0\% & 0\% & $2230 \pm 130$ & $2910 \pm 150$ & $3490 \pm 110$ \\
& 50  & $2.23 \pm 0.20$ & $2.53 \pm 0.27$ & $2.72 \pm 0.28$ & 0\% & 0\% & 0\% & $3000 \pm 110$ & $4780 \pm 85$ & $6280 \pm 72$ \\
& 75  & $5.8 \pm 0.6$ & $7.6 \pm 1.2$ & $7.8 \pm 0.6$ & 0\% & 0\% & 0\% & $3670 \pm 92$ & $6830 \pm 62$ & $9150 \pm 47$ \\
& 100 & $13.9 \pm 1.5$ & $14.8 \pm 1.7$ & $23.4 \pm 6.0$ & 0\% & 0\% & 0\% & $4270 \pm 120$ & $8800 \pm 110$ & $11990 \pm 81$ \\
\hline
\end{tabular}

\end{sidewaystable}

\color{black}
\subsection{Ablation Studies}
\label{sec:ablation}
This section evaluates the contribution of each algorithmic component and the sensitivity of Alg.~\ref{alg:integrated_path_planning} to key design choices. We examine three aspects: (1) the step-by-step contribution of each planning stage (Section~\ref{sec:ablation_steps}); (2) the behavior of the algorithm with arbitrary initial and final team positions (Section~\ref{sec:ablation_positions}); and (3) the empirical robustness-efficiency trade-off (Section~\ref{sec:ablation_robustness}). Unless otherwise specified, the same simulation setup as Section~\ref{sec:simulations} is used throughout.

\subsubsection{Impact of Algorithmic Steps}
\label{sec:ablation_steps}
To assess the contribution of each step of RSPECT to the final overall mission time, we conduct a per-step ablation where the algorithm is run with successive steps disabled or replaced by simpler alternatives. Three conditions are compared:

\begin{enumerate}
    \item[(i)] \textbf{TSP Lower Bound (Step 1):} The monitoring points are visited in the TSP order produced from Step~1 by the UAVs, ignoring energy limitations entirely. The resulting plan is infeasible in general, but its mission time serves as a lower-bound reference: the best achievable mission time if energy were not a limiting factor.
    
    \item[(ii)] \textbf{Feasible Tours (Steps 1 + 2.1, $c \in C_i^{\mu}$):} Tours are constructed by Step~2.1. For each tour $i$, we sweep over every candidate collect point in the feasible set $C_i^{\mu}$ and record the resulting mission time. The reported value is the mean $\pm$ standard deviation across this sweep, characterizing the distribution of mission times obtainable from valid collect-point selections in $C_i^{\mu}$.
    
    \item[(iii)] \textbf{Full Pipeline (Steps 1 + 2.1 + 2.2).} This is the complete RSPECT algorithm.
\end{enumerate}

The three conditions are evaluated across number of teams $m \in \{1, 2, 3, 4, 7, 10\}$ and number of monitoring points $n \in \{25, 50, 75, 100\}$. Each $(m, n)$ cell reports the mean and 
standard deviation of overall mission times over $25$ randomly generated $\mathcal{P}_\text{UAV}$ (uniform sampling from $\mathcal{Q}^a_f$).

Table~\ref{tab:ablation_raw} reports overall mission times in seconds for each condition. Table~\ref{tab:ablation_delta} reports the price of energy feasibility (\%-increase from TSP Lower Bound to Feasible Tours) and the gain of GTSP-based collect point selection (\%-decrease from Feasible Tours to Full Pipeline).

\begin{table}[htpb]
\centering
\color{black}
\small
\setlength{\tabcolsep}{4pt}
\caption{Overall Mission Times $\max_\mu\tau(X^\mu)$ for the per-step Ablation. TSP Lower Bound (TSP LB) is the Infeasible Lower Bound; Feasible Tours (Feas. Tour) and Full Pipeline (Full) Demonstrates Different Algorithmic Steps.}
\label{tab:ablation_raw}
\begin{tabular}{ccc cc ccc}
\toprule
 & & \multicolumn{3}{c}{$\boldsymbol{m=1}$} & \multicolumn{3}{c}{$\boldsymbol{m=2}$} \\
\cmidrule(lr){3-5}\cmidrule(lr){6-8}
 & $n$ & \textbf{TSP LB} & \textbf{Feas. Tour} & \textbf{Full} & \textbf{TSP LB} & \textbf{Feas. Tour} & \textbf{Full} \\
\midrule
& 25  & $1720 \pm 140$ & $5790 \pm 640$ & $5150 \pm 630$ & $1230 \pm 100$ & $4480 \pm 610$ & $4130 \pm 580$ \\
& 50  & $2380 \pm 120$ & $6610 \pm 560$ & $6250 \pm 680$ & $1530 \pm 70$  & $4490 \pm 440$ & $4070 \pm 390$ \\
& 75  & $2890 \pm 120$ & $7730 \pm 510$ & $7270 \pm 390$ & $1810 \pm 80$  & $5030 \pm 380$ & $4630 \pm 390$ \\
& 100 & $3310 \pm 110$ & $8570 \pm 540$ & $8200 \pm 540$ & $2070 \pm 90$  & $5510 \pm 420$ & $5160 \pm 410$ \\
\midrule
 & & \multicolumn{3}{c}{$\boldsymbol{m=3}$} & \multicolumn{3}{c}{$\boldsymbol{m=4}$} \\
\cmidrule(lr){3-5}\cmidrule(lr){6-8}
 & $n$ & \textbf{TSP LB} & \textbf{Feas. Tour} & \textbf{Full} & \textbf{TSP LB} & \textbf{Feas. Tour} & \textbf{Full} \\
\midrule
& 25  & $1010 \pm 100$ & $3940 \pm 490$ & $3650 \pm 560$ & $920 \pm 80$   & $3730 \pm 390$ & $3500 \pm 480$ \\
& 50  & $1200 \pm 70$  & $4090 \pm 340$ & $3710 \pm 340$ & $1060 \pm 70$  & $3910 \pm 360$ & $3530 \pm 430$ \\
& 75  & $1380 \pm 70$  & $4280 \pm 400$ & $3870 \pm 380$ & $1190 \pm 50$  & $4050 \pm 370$ & $3690 \pm 360$ \\
& 100 & $1530 \pm 80$  & $4500 \pm 480$ & $4100 \pm 480$ & $1300 \pm 80$  & $4200 \pm 420$ & $3830 \pm 470$ \\
\midrule
 & & \multicolumn{3}{c}{$\boldsymbol{m=7}$} & \multicolumn{3}{c}{$\boldsymbol{m=10}$} \\
\cmidrule(lr){3-5}\cmidrule(lr){6-8}
 & $n$ & \textbf{TSP LB} & \textbf{Feas. Tour} & \textbf{Full} & \textbf{TSP LB} & \textbf{Feas. Tour} & \textbf{Full} \\
\midrule
& 25  & $830 \pm 30$   & $3570 \pm 180$ & $3420 \pm 210$ & $830 \pm 30$   & $3570 \pm 180$ & $3430 \pm 200$ \\
& 50  & $890 \pm 40$   & $3730 \pm 170$ & $3510 \pm 270$ & $850 \pm 20$   & $3720 \pm 180$ & $3480 \pm 220$ \\
& 75  & $970 \pm 40$   & $3860 \pm 280$ & $3580 \pm 390$ & $880 \pm 20$   & $3750 \pm 170$ & $3530 \pm 240$ \\
& 100 & $1000 \pm 50$  & $3940 \pm 300$ & $3580 \pm 440$ & $910 \pm 20$   & $3820 \pm 180$ & $3610 \pm 290$ \\
\bottomrule
\end{tabular}
\end{table}
\color{black}

\begin{table}[htpb]
\centering
\color{black}
\small
\setlength{\tabcolsep}{4pt}
\caption{Price of Energy Feasibility ($\Delta_{\text{feas}}\%$: \%-increase from TSP Lower Bound to Feasible Tours) and Gain from GTSP-based selection ($\Delta_{\text{GTSP}}\%$: \%-decrease from Feasible Tours to Full Pipeline).}
\label{tab:ablation_delta}
\begin{tabular}{ccc cc cc c} 
\toprule
 & & \multicolumn{2}{c}{$\boldsymbol{m=1}$} & \multicolumn{2}{c}{$\boldsymbol{m=2}$} & \multicolumn{2}{c}{$\boldsymbol{m=3}$} \\
\cmidrule(lr){3-4}\cmidrule(lr){5-6}\cmidrule(lr){7-8}
 & $n$ & \textbf{$\Delta_{\text{feas}}\%$} & \textbf{$\Delta_{\text{GTSP}}\%$} & \textbf{$\Delta_{\text{feas}}\%$} & \textbf{$\Delta_{\text{GTSP}}\%$} & \textbf{$\Delta_{\text{feas}}\%$} & \textbf{$\Delta_{\text{GTSP}}\%$} \\
\midrule
& 25  & $+236.6\%$ & $-11.1\%$ & $+264.2\%$ & $-7.8\%$ & $+290.1\%$ & $-7.4\%$ \\
& 50  & $+177.7\%$ & $-5.5\%$  & $+193.5\%$ & $-9.4\%$ & $+240.8\%$ & $-9.3\%$ \\
& 75  & $+167.5\%$ & $-5.9\%$  & $+177.9\%$ & $-8.0\%$ & $+210.1\%$ & $-9.6\%$ \\
& 100 & $+158.9\%$ & $-4.3\%$  & $+166.2\%$ & $-6.4\%$ & $+194.1\%$ & $-8.9\%$ \\
\midrule
 & & \multicolumn{2}{c}{$\boldsymbol{m=4}$} & \multicolumn{2}{c}{$\boldsymbol{m=7}$} & \multicolumn{2}{c}{$\boldsymbol{m=10}$} \\
\cmidrule(lr){3-4}\cmidrule(lr){5-6}\cmidrule(lr){7-8}
 & $n$ & \textbf{$\Delta_{\text{feas}}\%$} & \textbf{$\Delta_{\text{GTSP}}\%$} & \textbf{$\Delta_{\text{feas}}\%$} & \textbf{$\Delta_{\text{GTSP}}\%$} & \textbf{$\Delta_{\text{feas}}\%$} & \textbf{$\Delta_{\text{GTSP}}\%$} \\
\midrule
& 25  & $+305.4\%$ & $-6.2\%$ & $+330.1\%$ & $-4.2\%$ & $+330.1\%$ & $-3.9\%$ \\
& 50  & $+268.9\%$ & $-9.7\%$ & $+319.1\%$ & $-5.9\%$ & $+337.6\%$ & $-6.5\%$ \\
& 75  & $+240.3\%$ & $-8.9\%$ & $+297.9\%$ & $-7.3\%$ & $+326.1\%$ & $-5.9\%$ \\
& 100 & $+223.1\%$ & $-8.8\%$ & $+294.0\%$ & $-9.1\%$ & $+319.8\%$ & $-5.5\%$ \\
\bottomrule
\end{tabular}
\end{table}

\color{black}
Two observations follow from Tables~\ref{tab:ablation_raw} and~\ref{tab:ablation_delta}. First, the 
energy-feasibility cost $\Delta_{\text{feas}}$ is large (around 150-340\% across all $(m, n)$), confirming that the explicit feasibility construction in Step~2.1 is responsible for the dominant share of mission time relative to a non-energy-aware tour. Second, $\Delta_{\text{GTSP}}$ is consistently negative across all $(m, n)$ tested (3-11\%), indicating that the GTSP-based selection in Step~2.2 consistently finds collect-point combinations that lie below the mean of mission times obtainable from valid choices in $C_i^{\mu}$. The $\Delta_{\text{GTSP}}$ gap is more pronounced at lower $m$ and moderate $n$, where the GTSP solver has more inter-tour coupling to exploit; at higher $m$ where each team has fewer tours, the gap narrows but remains consistently in RSPECT's favor.

\subsubsection{Sensitivity to Initial and Final Team Positions}
\label{sec:ablation_positions}

To assess the behavior of Alg.~\ref{alg:integrated_path_planning} to the spatial configuration 
of team start and end positions, we test under a configuration where $p_o^\mu$ and $p_f^\mu$ are sampled uniformly from $\mathcal{Q}_f^g$ for each team and each instance. 
This complements the scalability analysis in Section~\ref{subsec:scalability}, which used fixed $p_o^\mu$ and $p_f^\mu$. Evaluations are made across number of teams $m \in \{1, 2, 3, 4, 7, 10\}$ and number of monitoring points $n \in \{25, 50, 75, 100\}$. Each $(m, n)$ cell reports the mean and 
standard deviation of mission times over $25$ randomly generated $\mathcal{P}_\text{UAV}$ (uniform sampling from $\mathcal{Q}^a_f$).

Fig.~\ref{fig:ablation_positions} illustrates the algorithm output for a representative instance under this configuration with $m=4$ teams. Fig. \ref{fig:ablation_positions_ascending} illustrates the plans for the same representative configuration with ascending number of teams i.e., $m \in \{1,2,3,4\}$. Results across all tested configurations are summarized in 
Table~\ref{tab:ablation_positions}.

\begin{figure}[htpb]
    \centering
    \color{black}
    \includegraphics[width=\linewidth]{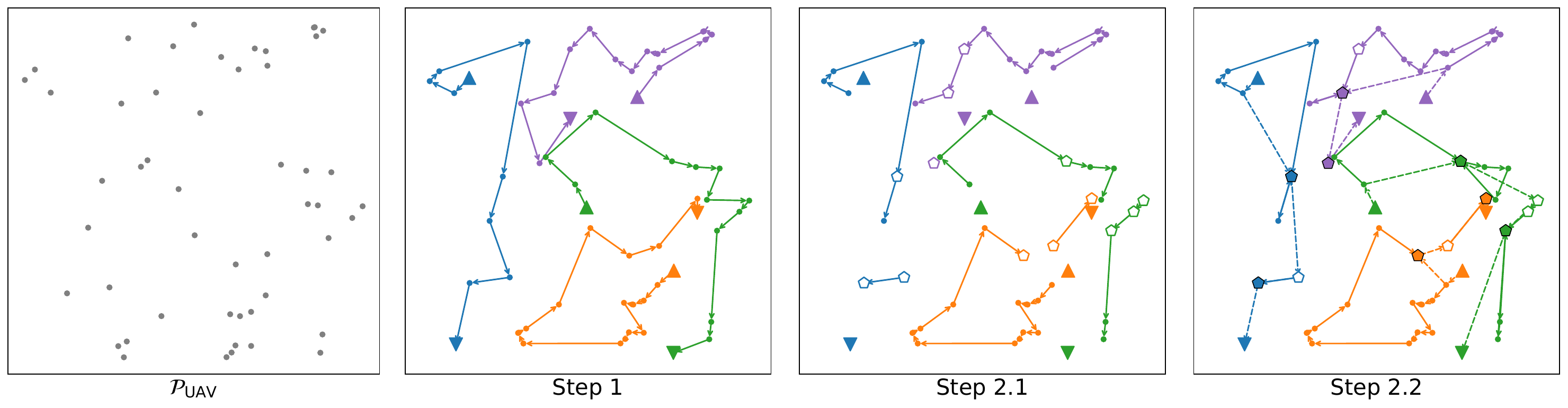}
    \caption{Output of Alg.~\ref{alg:integrated_path_planning} for 
    a representative instance under randomly generated 
    $p_o^\mu$/$p_f^\mu$ with $m=4$ teams (top view). Left: 
    $\mathcal{P}_{\text{UAV}}$ (gray points). Step~1: Assigned points and sequences, with color-coded
    upward/downward triangles showing $p_o^\mu$/$p_f^\mu$ for each team. 
    Step~2.1: Energy-feasible tours with feasible collect point candidates (hollow pentagons) 
    and Step 2.2: final plan with selected collect points (filled pentagons). Color-coded dashed lines are UGV paths.}
    \label{fig:ablation_positions}
\end{figure}

\begin{figure}
    \centering
    \color{black}
    \includegraphics[width=1.0\linewidth]{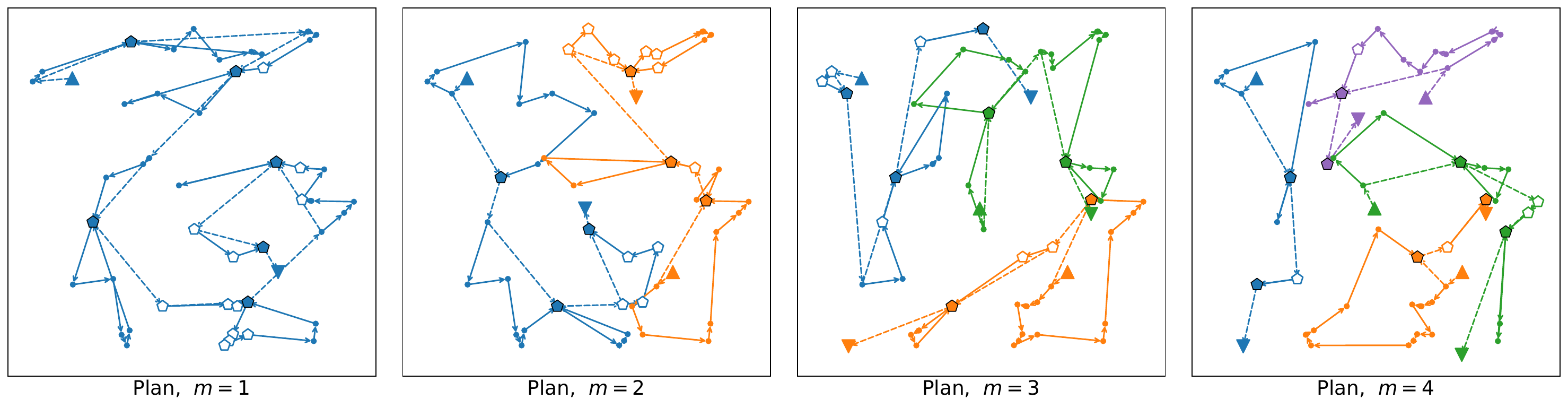}
    \caption{Output of Alg.~\ref{alg:integrated_path_planning} for 
    a representative instance under randomly generated 
    $p_o^\mu$/$p_f^\mu$ with $m\in \{1,2,3,4\}$ teams (top view). Teams' plans are color coded, hollow pentagons are feasible collect points, filled pentagons are selected collect points, and dashed lines are the UGV paths.}
    \label{fig:ablation_positions_ascending}
\end{figure}

Computation times are consistent with those reported in 
Table~\ref{tab:heuristic_performance}, confirming that the 
$\mathcal{O}(mn + n^3)$ complexity is independent of the team 
initial/final position configuration. Mission times are also comparable across 
problem sizes and team counts, demonstrating that 
Alg.~\ref{alg:integrated_path_planning} generalizes well to 
arbitrary start and end configurations without any modification to 
the algorithm.

\begin{table}[htpb]
\centering
\color{black}
\small
\setlength{\tabcolsep}{4pt}
\caption{Computation Time and Overall Mission Time $\max_\mu\tau(X^\mu)$ Under Randomly Generated Team Initial ($p_o^\mu$) and Final Positions ($p_o^\mu$), for different Number of Teams $m$ and Number of Monitoring Points $n$.}
\label{tab:ablation_positions}
\begin{tabular}{ccc cc cc}
\toprule
 & & \multicolumn{2}{c}{$\boldsymbol{m=1}$} & \multicolumn{2}{c}{$\boldsymbol{m=2}$} \\
\cmidrule(lr){3-4}\cmidrule(lr){5-6}
 & $n$ & \textbf{Comp. (s)} & $\boldsymbol{\max_{\mu}\tau(X^{\mu})}$ \textbf{(s)}
       & \textbf{Comp. (s)} & $\boldsymbol{\max_{\mu}\tau(X^{\mu})}$ \textbf{(s)} \\
\midrule
& 25  & $0.08 \pm 0.01$ & $4830 \pm 880$ & $0.26 \pm 0.01$ & $3060 \pm 760$ \\
& 50  & $0.21 \pm 0.03$ & $6120 \pm 640$ & $1.03 \pm 0.05$ & $3730 \pm 630$ \\
& 75  & $0.45 \pm 0.08$ & $6860 \pm 430$ & $2.91 \pm 0.20$ & $4030 \pm 470$ \\
& 100 & $0.82 \pm 0.13$ & $7810 \pm 400$ & $6.83 \pm 0.38$ & $4590 \pm 430$ \\
\midrule
 & & \multicolumn{2}{c}{$\boldsymbol{m=3}$} & \multicolumn{2}{c}{$\boldsymbol{m=4}$} \\
\cmidrule(lr){3-4}\cmidrule(lr){5-6}
 & $n$ & \textbf{Comp. (s)} & $\boldsymbol{\max_{\mu}\tau(X^{\mu})}$ \textbf{(s)}
       & \textbf{Comp. (s)} & $\boldsymbol{\max_{\mu}\tau(X^{\mu})}$ \textbf{(s)} \\
\midrule
& 25  & $0.30 \pm 0.02$ & $2380 \pm 510$ & $0.36 \pm 0.02$ & $2170 \pm 590$ \\
& 50  & $0.92 \pm 0.04$ & $2950 \pm 560$ & $0.93 \pm 0.03$ & $2630 \pm 620$ \\
& 75  & $2.27 \pm 0.12$ & $3110 \pm 490$ & $2.06 \pm 0.09$ & $2590 \pm 440$ \\
& 100 & $4.83 \pm 0.34$ & $3480 \pm 480$ & $4.02 \pm 0.15$ & $2790 \pm 490$ \\
\midrule
 & & \multicolumn{2}{c}{$\boldsymbol{m=7}$} & \multicolumn{2}{c}{$\boldsymbol{m=10}$} \\
\cmidrule(lr){3-4}\cmidrule(lr){5-6}
 & $n$ & \textbf{Comp. (s)} & $\boldsymbol{\max_{\mu}\tau(X^{\mu})}$ \textbf{(s)}
       & \textbf{Comp. (s)} & $\boldsymbol{\max_{\mu}\tau(X^{\mu})}$ \textbf{(s)} \\
\midrule
& 25  & $0.52 \pm 0.04$ & $2040 \pm 440$ & $0.63 \pm 0.04$ & $1770 \pm 180$ \\
& 50  & $1.14 \pm 0.05$ & $2170 \pm 490$ & $1.36 \pm 0.04$ & $2080 \pm 450$ \\
& 75  & $2.10 \pm 0.07$ & $2410 \pm 500$ & $2.38 \pm 0.08$ & $2240 \pm 490$ \\
& 100 & $3.55 \pm 0.13$ & $2390 \pm 550$ & $3.75 \pm 0.14$ & $2100 \pm 430$ \\
\bottomrule
\end{tabular}
\end{table}
\color{black}
\color{black}
\subsubsection{Sensitivity to Robustness Parameters}
\label{sec:ablation_robustness}

To empirically evaluate the robustness-efficiency trade-off, we vary the robustness parameters 
$\delta_a, \delta_g \in \{0, 50, 100, 200\}\,\text{s}$ and measure the resulting $\max_\mu\tau(X^\mu)$ for $n=100$ and $m \in \{1, 2, 3, 4, 7, 10\}$ with 25 randomly generated $\mathcal{P}_{\text{UAV}}$ (uniform sampling from $\mathcal{Q}^a_f$). Recall that $\delta_a$ and $\delta_g$ impose lower bounds 
on the tour robustness of every tour (Definition~\ref{def1}): larger values force the planner to reserve more slack in each tour, shrinking the feasible space and potentially increasing mission time. The 
parameter configuration $\delta_a = \delta_g = 0$ corresponds to planning without any robustness margin.

Results are summarized in Table~\ref{tab:ablation_robustness}. As expected, mission time increases with the slack level for all configurations. However, the increase is modest: even at $\max(\delta_a,\delta_g) = 200\,\text{s}$ (a third of the maximum flight time $\overline{\tau_a} = 600\,\text{s}$)
mission time increases by at most $14\%$ for $m=1$ and less than $5\%$ for $m \geq 7$ relative to the no-slack baseline. This confirms that substantial robustness margins can be achieved at a small cost in 
mission efficiency, particularly when multiple teams are available to share the workload. 

\begin{table}[htpb]
\centering
\color{black}
\small
\setlength{\tabcolsep}{4pt}
\caption{Overall Mission Time, $\max_\mu\tau(X^\mu)$ for Different 
Robustness Parameters Configurations, $\max(\delta_a,\delta_g)$ and Number of 
Teams $m$, with $n=100$.}
\label{tab:ablation_robustness}
\begin{tabular}{cc cccc}
\toprule
$m$ &\text{$\max(\delta_a,\delta_g)$=0 s} & \text{$\max(\delta_a,\delta_g)$=50 s} & 
\text{$\max(\delta_a,\delta_g)$=100 s} & \text{$\max(\delta_a,\delta_g)$=200 s} \\
\midrule
1  & $8200 \pm 540$ & $8370 \pm 520$ & $8710 \pm 400$ & $9380 \pm 680$ \\
2  & $5160 \pm 410$ & $5400 \pm 480$ & $5530 \pm 510$ & $6120 \pm 600$ \\
3  & $4100 \pm 480$ & $4330 \pm 420$ & $4420 \pm 470$ & $4810 \pm 450$ \\
4  & $3830 \pm 470$ & $3960 \pm 550$ & $4200 \pm 610$ & $4500 \pm 490$ \\
7  & $3580 \pm 440$ & $3760 \pm 420$ & $3810 \pm 380$ & $3880 \pm 240$ \\
10 & $3610 \pm 290$ & $3660 \pm 240$ & $3650 \pm 280$ & $3690 \pm 180$ \\
\bottomrule
\end{tabular}
\end{table}

\color{black}

\color{black}

\section{Experiments}
Experiments were performed with two teams, each consisting of a Crazyflie 2.1 (UAV) and a Turtlebot4 (UGV), to demonstrate the applicability of the proposed approach in a multi-team setting (see Fig. \ref{fig:rspect_experiment_snapshot}). An OptiTrack motion capture system covering a volume of $5\,\text{m} \times 5\,\text{m} \times 3\,\text{m}$ was used for real-time localization. The software architecture was built on ROS2, with motion capture streaming real-time pose data for all robots. UGVs navigated using motion capture localization, and UAVs were controlled via takeoff, land, and waypoint navigation commands. A Kalman filter handled UAV state estimation and a Mellinger controller managed flight control. 

\begin{figure}[htpb]
    \centering
    \color{black}
    \includegraphics[width=1.0\linewidth]{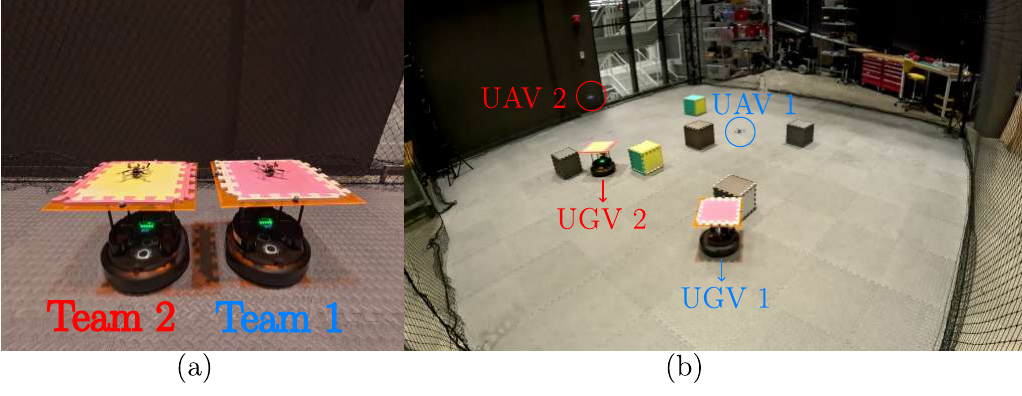}
    \caption{Experimental setup. (a) Both UAV-UGV teams, each consisting a Crazyflie 2.1 UAV resting on a Turtlebot4 UGV. (b) A snapshot of the experiment, with unknown ground obstacles (colored boxes) placed in the environment.}
    \label{fig:rspect_experiment_snapshot}
\end{figure}

The maximum flight time was set to $\overline{\tau_a} = 60\,\text{s}$, with robustness parameters $\delta_a = \delta_g = 15\,\text{s}$. A fixed recharging time of $5$ s was used.

The scenario had 20 monitoring points in $\mathcal{P}_{\text{UAV}}$. Alg.\ref{alg:integrated_path_planning} was used to generate an offline plan without any knowledge of the obstacles. Each teams' plan had 2 tours. The offline plan, its overlay with the unknown environment, and the actual online execution are shown in Fig.\ref{fig:rspect_experiment_plan}(a), (b), and (c) respectively.

\begin{figure}[htpb]
    \centering
    \color{black}
    \includegraphics[width=1.0\linewidth]{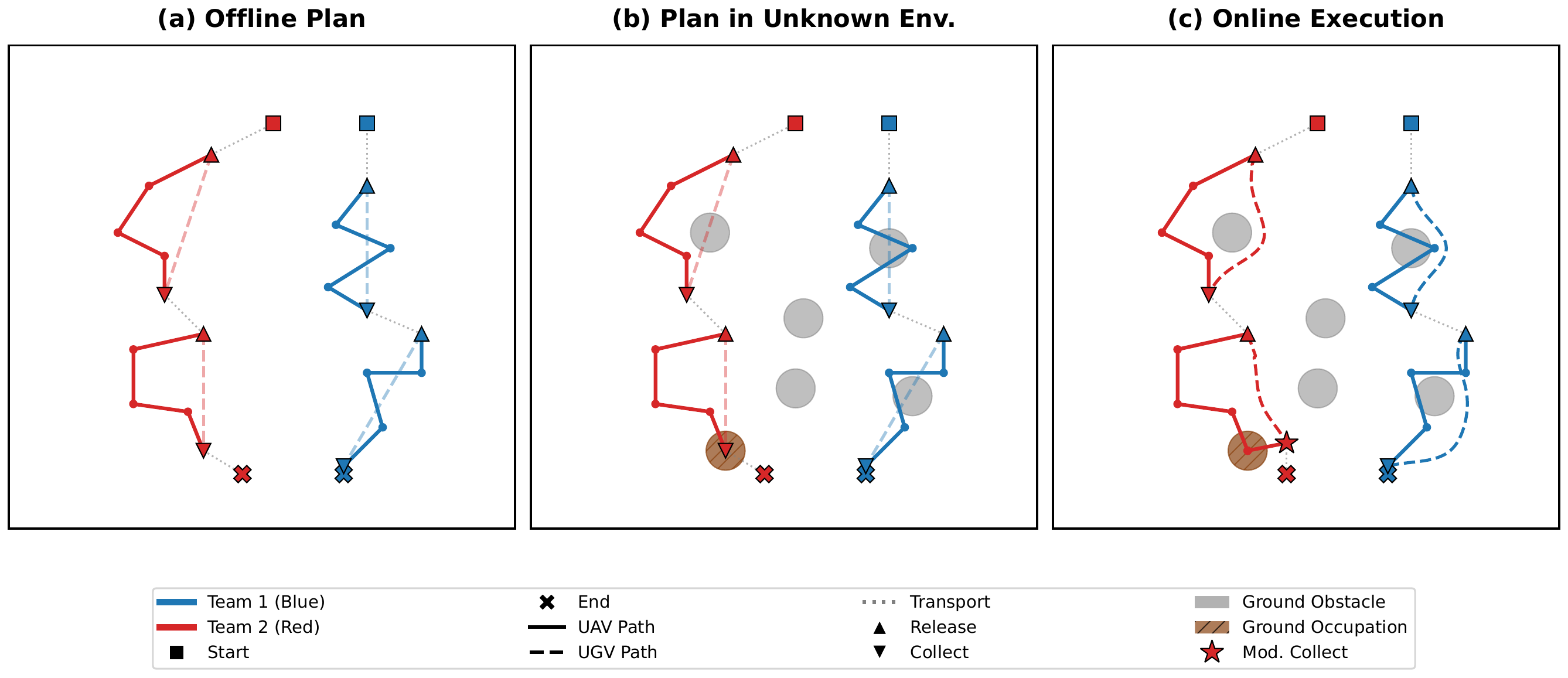}
    \caption{Plan of the two-team mission. (a) Offline plan generated by Alg. \ref{alg:integrated_path_planning} prior to execution in an obstacle-free environment. (b) The same offline plan overlaid with the unknown obstacles discovered during execution. (c) Online execution showing UGV detours around unknown obstacles and the collect point relocation event for Team 2's second tour: the planned collect point (red downward triangle) was occupied and relocated to an alternative (red $\star$),}
    \label{fig:rspect_experiment_plan}
\end{figure}

As the teams followed the plan, unknown ground obstacles were encountered by the UGVs, causing path deviations. In addition, Team 2's planned collect point in Tour 2 was found to be occupied during execution. Following Corollary \ref{cor}, an alternative collect point was selected.

Table~\ref{tab:experiment_summary} summarizes the per-tour recorded quantities for both teams. All realized UAV and UGV travel times remain within $\overline{\tau_a} = 60\,\text{s}$.

The realized tour robustness values $\hat{\delta}_a^{\mu,i}$ and $\hat{\delta}_g^{\mu,i}$ in Table~\ref{tab:experiment_summary} illustrate the role of the design parameters $\delta_a$ and $\delta_g$. 
By constraints \eqref{eq:con2}-\eqref{eq:con3}, the offline plan $X$ is constructed such that every tour satisfies $\hat{\delta}_a^{\mu,i}(X) \geq \delta_a$ and $\hat{\delta}_g^{\mu,i}(X) \geq \delta_g$ (see Def. \ref{def1}); the design parameters impose lower bounds on the per-tour robustness at planning time. At execution time, the realized values $\hat{\delta}_a^{\mu,i}$ and $\hat{\delta}_g^{\mu,i}$ can be either above or below the design margin depending on how much of the budget each disturbance consumed.

For $(\mu, i) \in \{(1,1), (1,2), (2,1)\}$, the realized $\hat{\delta}_a^{\mu,i}$ exceeds the design $\delta_a = 15\,\text{s}$, indicating that the actual UAV traversal consumed less than the design 
budget; similarly, all realized $\hat{\delta}_g^{\mu,i}$ values exceed $\delta_g = 15\,\text{s}$ by a wide margin. These tours demonstrate that \eqref{eq:con2}-\eqref{eq:con3} act as lower bounds rather than 
exact predictions of the realized robustness.

For $(\mu, i) = (2, 2)$, the planned collect point was found to be occupied during execution, and the local adjustment mechanism of Corollary \ref{cor} was invoked. Corollary \ref{cor} bounds the total release and collect deviations; in our implementation, this budget is split equally between release and collect, giving a per-point displacement bound of $v_H^{\text{avg}} \cdot \hat{\delta}_a^{2,2}(X) / 2 = 3.20\,\text{m}$, where $\hat{\delta}_a^{2,2}(X) \approx 16\,\text{s}$ is evaluated on the offline plan. The realized displacement was $0.51\,\text{m}$, well within this bound. The relocation itself consumed approximately $9\,\text{s}$ of additional UAV flight time, which is why the realized $\hat{\delta}_a^{2,2} = 7\,\text{s}$ in Table~\ref{tab:experiment_summary} is below the design $\delta_a = 15\,\text{s}$.

\begin{table}[htpb]
\centering
\color{black}
\caption{Experiment logs. UAV time and UGV time are the realized 
traversal times under execution. $\hat{\delta}_a^{\mu,i}$ and 
$\hat{\delta}_g^{\mu,i}$ are the realized per-tour robustness values.}
\label{tab:experiment_summary}
\begin{tabular}{|c|c|c|c|c|}
\hline
\textbf{$(\mu, i)$} & \textbf{UAV time (s)} & \textbf{UGV time (s)} & 
\textbf{$\hat{\delta}_a^{\mu,i}$ (s)} & \textbf{$\hat{\delta}_g^{\mu,i}$ (s)} \\ \hline
$(1, 1)$ & $42$ & $11.9$ & $18$ & $48.1$ \\ \hline
$(1, 2)$ & $43$ & $14.7$ & $17$ & $45.3$ \\ \hline
$(2, 1)$ & $43$ & $13.3$ & $17$ & $46.7$ \\ \hline
$(2, 2)$ & $53$ & $8.5$  & $7$  & $51.5$ \\ \hline
\end{tabular}
\end{table}

The experiment video is available in the supplementary materials and online.\footnote{Video link: \url{https://youtu.be/jQkGfwVTdBc}}

\color{black}

\color{black}
\section{Conclusion}
\label{sec:conclusion}

We presented RSPECT, a scalable and efficient heuristic algorithm for robust  planning of energy-constrained UAVs and mobile charging stations (UGVs) to perform long-horizon aerial monitoring missions. We provided theoretical results on the complexity and robustness of our approach. To our knowledge, this combination of formal complexity and robustness guarantees has not been previously achieved for this type of energy-aware UAV-UGV planning problems. These theoretical results were also supported by simulations that demonstrated the performance of our approach and compared it to some standard optimization techniques as well as other related works (\cite{ropero2019terra,yu2018algorithms,maini2019coverage}). Experiments were also provided for validation. 

\color{black}
Future work includes extending our approach to monitoring missions with more generic temporal constraints (e.g., visiting the air points periodically, in a specific order or within specific time windows). In addition, as an alternative to the deterministic robustness margins used in \eqref{opt_problem}, we plan to explore stochastic formulations where the energy-feasibility imposed by \eqref{eq:con2}, \eqref{eq:con3} is instead encoded as chance constraints. Furthermore, we also plan to extend our approach to other long-horizon missions such as patrolling, SLAM, and search. Moreover, we plan to enhance our approach with more on-board autonomy (e.g., for energy monitoring/harvesting, reactive planning and coordination, dynamic environment/obstacles). Finally, we plan to incorporate learning-based approaches for improving the overall performance in stochastic scenarios.
\color{black}

\section*{Data Availability}
The authors declare that no datasets were generated or analyzed during this study.

\color{black}

\section*{Nomenclature}
\begin{longtable}{p{0.18\linewidth} p{0.62\linewidth} p{0.12\linewidth}}
\hline
\textbf{Symbol} & \textbf{Definition} & \textbf{First use} \\
\hline
\endfirsthead
\hline
\textbf{Symbol} & \textbf{Definition} & \textbf{First use} \\
\hline
\endhead
\hline
\endfoot

\multicolumn{3}{l}{\textit{Environment}} \\ \hline
$\mathcal{Q}$ & Environment & Eq. \eqref{env} \\
$\mathcal{Q}_f$ & Feasible space & Sec. \ref{sec:problem_formulation} \\
$\mathcal{Q}_f^g$ & Feasible ground space ($z=0$ slice of $\mathcal{Q}_f$) & Sec. \ref{sec:problem_formulation}\\
$\mathcal{Q}_f^a$ & Feasible air space ($\mathcal{Q}_f \setminus \mathcal{Q}_f^g$) & Sec.~\ref{sec:problem_formulation} \\
$\tilde{\mathcal{Q}}_f$ & Actual feasible space & Sec.~\ref{subsec:robustness} \\
$\tilde{\mathcal{Q}}_f^g$ & Actual feasible ground space & Sec.~\ref{subsec:robustness} \\
$\ell(\cdot,\cdot)$ & Shortest-path distance function within $\mathcal{Q}_f$ & Sec.~\ref{sec:problem_formulation} \\
$\underline{z}$ & Minimum UAV flight altitude (Assumption~1) & Sec.~\ref{sec:vtol} \\ \hline

\multicolumn{3}{l}{\textit{Mission}} \\ \hline
$\mathcal{P}_{\text{UAV}}$ & Set of $n$ aerial monitoring points & Sec. \ref{sec:problem_formulation} \\
$n$ & Number of monitoring points ($n = |\mathcal{P}_{\text{UAV}}|$) & Sec. \ref{sec:problem_formulation} \\
$m$ & Number of UAV--UGV teams & Sec. \ref{sec:problem_formulation} \\
$\mu$ & Team index ($\mu \in \{1,\ldots,m\}$) & Sec. \ref{sec:problem_formulation} \\
$p_o^\mu$ & Initial (start) position of team $\mu$ & Sec. \ref{sec:problem_formulation} \\
$p_f^\mu$ & Final (end) position of team $\mu$ & Sec. \ref{sec:problem_formulation} \\
$\mathcal{P}_{\text{UAV}}^\mu$ & Monitoring points assigned to team $\mu$ & Sec.~\ref{subsec:proposed_algorithm} \\
$n_\mu$ & Number of points assigned to team $\mu$ ($n_\mu = |\mathcal{P}_{\text{UAV}}^\mu|$) & Sec.~\ref{subsec:proposed_algorithm} \\ \hline

\multicolumn{3}{l}{\textit{Mission plan}} \\ \hline
$X^\mu$ & Mission plan matrix for team $\mu$ & Sec.~\ref{sec:problem_formulation} \\
$X^\mu_i$ & $i$-th row of $X^\mu$ (tour $i$ of team $\mu$) & Sec.~\ref{sec:problem_formulation} \\
$X^\mu_{i,j}$ & $j$-th entry of tour $i$ of team $\mu$ & Sec.~\ref{sec:problem_formulation} \\
$X^\mu_{i,1}$ & Release point of tour $i$ of team $\mu$ & Sec.~\ref{sec:problem_formulation} \\
$X^\mu_{i,n_\mu+2}$ & Collect point of tour $i$ of team $\mu$ & Sec.~\ref{sec:problem_formulation} \\
$X$ & Overall mission plan $\{X^1,\ldots,X^m\}$ & Sec.~\ref{sec:problem_formulation} \\
$\tilde{X}^\mu$ & Modified plan for team $\mu$ & Sec.~\ref{subsec:robustness} \\
$\tilde{X}$ & Modified overall plan $\{\tilde{X}^1,\ldots,\tilde{X}^m\}$ & Sec.~\ref{subsec:robustness} \\ \hline

\multicolumn{3}{l}{\textit{Time functions}} \\ \hline
$\tau_a(\cdot)$ & UAV travel time function & Sec.~\ref{sec:problem_formulation} \\
$\tau_g(\cdot,\cdot)$ & UGV travel time function & Sec.~\ref{sec:problem_formulation} \\
$\tau_c(\cdot)$ & UAV recharge time function & Sec.~\ref{sec:problem_formulation} \\
$\tau(X^\mu)$ & Total execution time of team $\mu$'s plan & Eq.~(\ref{eq:objective_function}) \\
$\overline{\tau_a}$ & Maximum UAV flight time (endurance limit) & Sec.~\ref{sec:problem_formulation} \\
$\gamma$ & Recharge ratio & Eq.~\eqref{eq:linear_charging_equation} \\
$\tilde{\tau}_a(\cdot)$ & Actual UAV traversal time & Sec.~\ref{subsec:robustness}\\
$\tilde{\tau}_g(\cdot,\cdot)$ & Actual UGV traversal time  & Sec.~\ref{subsec:robustness} \\
$\tau_{\text{TL}}$ & UAV takeoff/landing time (constant under Assmp.~\ref{assump1}) & Sec.~\ref{sec:vtol} \\ \hline

\multicolumn{3}{l}{\textit{Robustness}} \\ \hline
$\delta_a$ & Aerial robustness parameter & Eq.~(\ref{eq:con2}) \\
$\delta_g$ & Ground robustness parameter  & Eq.~(\ref{eq:con3}) \\
$\hat{\delta}_a^{\mu,i}(X)$ & Tour robustness of aerial leg & Def.~\ref{def1} \\
$\hat{\delta}_g^{\mu,i}(X)$ & Tour robustness of ground leg & Def.~\ref{def1} \\
$v_H^{\text{avg}}$ & Average UAV horizontal speed & Cor.~\ref{cor} \\
$v_g^{\text{avg}}$ & Average UGV ground speed & Cor.~\ref{cor} \\
$\ell_{g,i}^\mu$ & Length of shortest feasible ground path between release and collect of team $\mu$'s tour $i$ under $X$ & Cor.~\ref{cor} \\
$\tilde{\ell}_{g,i}^\mu$ & Length of shortest feasible ground path between release and collect of team $\mu$'s tour $i$ under $\tilde{X}$ & Cor.~\ref{cor} \\ 
$\Delta(\cdot,\cdot)$ & Minimum UAV travel time between two points at fixed altitude & Cor.~\ref{cor}\\\hline

\multicolumn{3}{l}{\textit{Algorithm}} \\ \hline
$\mathcal{S}^\mu$ & TSP sequence for team $\mu$ (permutation of $\mathcal{P}_{\text{UAV}}^\mu$) & Sec.~\ref{subsec:proposed_algorithm} \\
$\mathcal{C}_i^\mu$ & Set of feasible collect-point candidates for tour $i$ of team $\mu$ & Sec.~\ref{subsec:proposed_algorithm} \\
$\pi(\cdot)$ & Projection function onto $\mathcal{Q}_f^g$ & Sec.~\ref{subsec:proposed_algorithm} \\
$\sigma$ & GLNS runtime constant (time limit $= \sigma n_\mu^3$) & Sec.~\ref{subsec:proposed_algorithm} \\
 \hline

\end{longtable}

\color{black}
\bibliography{references}
\color{black}
\appendix
\section{Baseline Encodings}
\label{app:baselines}

This appendix describes the encoding of problem~\eqref{opt_problem} for the Branch and Cut, SA, and GA baselines used in Section~\ref{sec:simulations}. The three literature heuristics \cite{ropero2019terra,yu2018algorithms,maini2019coverage} are used as published. All methods share the same travel-time instantiation, platform speeds, and recharge model (Section~\ref{sec:simulations}); SA and GA hyperparameter details are in Sections~\ref{sec:sa} and~\ref{sec:ga}.

\subsection{Branch and Cut}
\label{app:bnc}
Problem~\eqref{opt_problem} is presented in a problem-natural form using logical quantifiers, non-smooth $\max(\cdot)$ terms, and Euclidean distances. For exact-solver evaluation, we encode it as a mixed-integer
quadratically constrained program (MIQCP) with the following decision variables:
\begin{itemize}
    \item $X^{\mu}_{i,j} \in \mathbb{R}^3$: continuous point position for team $\mu$, tour $i$, slot $j$;
    \item $z_{\mu,i,j,k} \in \{0,1\}$: binary indicator that slot $j$ of tour $i$ of team $\mu$ corresponds to monitoring point $p_k$;
    \item $T_{\max} \in \mathbb{R}_{\geq 0}$: makespan epigraph variable.
\end{itemize}

The constraints are encoded as follows:
\begin{itemize}
\item \textit{Objective~\eqref{eq:eq1}:} $T_{\max} \geq \tau(X^{\mu})$ for 
all $\mu$; minimize $T_{\max}$.

    \item \textit{Constraint ~\eqref{eq:con1}:} 
$\sum_{\mu,i,j} z_{\mu,i,j,k} \geq 1$ for each $k$, ensuring every 
monitoring point is visited by at least one team.

    \item \textit{Position assignment (Big-$M$):} When $z_{\mu,i,j,k} = 1$, slot 
    $j$ must equal $p_k$; enforced via Big-$M$ constraints 
    $|X^{\mu}_{i,j,d} - p_{k,d}| \leq M(1 - z_{\mu,i,j,k})$ for each 
    coordinate $d \in \{x, y, z\}$.

\item \textit{Constraints \eqref{eq:con2} and \eqref{eq:con3}:} Enforced 
directly via $\tau_a(X^{\mu}_i) + \delta_a \leq \overline{\tau_a}$ and 
$\tau_g(X^{\mu}_{i,1}, X^{\mu}_{i,n_\mu+2}) + \delta_g \leq 
\overline{\tau_a}$. The $\max(\cdot)$ terms and 
Euclidean distance terms are handled using Gurobi's built-in support for 
general constraints and non-convex quadratic constraints respectively..

\item \textit{Constraint ~\eqref{eq:con4}:} Release and collect slots are 
constrained to the ground plane ($z = 0$).

\item \textit{Constraint ~\eqref{eq:con5}:} Each mid-tour slot is encoded 
as $X^{\mu}_{i,j} = \sum_k z_{\mu,i,j,k}\, p_k + (1 - \sum_k 
z_{\mu,i,j,k})\, X^{\mu}_{i,j-1}$, which reduces to copying the previous 
slot when no point is assigned.

\end{itemize}

The model is solved using Gurobi~\cite{gurobi} with stopping criteria 
\texttt{MipGap} $= 5\%$ and \texttt{TimeLimit} $= 200\,\text{s}$. 
Numerical tolerances are set to \texttt{FeasibilityTol} $= 10^{-4}$, 
\texttt{IntFeasTol} $= 10^{-5}$, and \texttt{OptimalityTol} $= 10^{-2}$; 
non-convex quadratic constraints are enabled via \texttt{NonConvex} $= 2$. 
A feasible warm start is provided via round-robin assignment of 
monitoring points to teams, with release and collect points placed 
directly below each assigned monitoring point, reducing branch-and-bound 
search time.

\subsection{Simulated Annealing}

The SA state is a list of $m$ matrices, one per team, each of shape $n \times (n+2)$, matching the structure of $X^{\mu}$ in \eqref{opt_problem}. The initial state is constructed by the same round-robin assignment used to warm-start the MIQCP.

The SA energy combines the makespan with structural penalty terms that encode the soft constraints: a large penalty for each tour violating the constraints~\eqref{eq:con2}-\eqref{eq:con3}, and a coverage penalty for each monitoring point in $\mathcal{P}_{\text{UAV}}$ not visited by any team. The remaining constraints, \eqref{eq:con4} and \eqref{eq:con5}, are enforced as hard constraints by the state representation and the move operator, which only generate states satisfying them by construction.

The neighborhood operator is structure-aware and includes both intra-team and cross-team moves: intra-team release/collect-point displacement, intra-team point resampling from $\mathcal{P}_{\text{UAV}}$, intra-team point repetition, cross-team point reassignment, and cross-team point swap. The choice among these moves and number of choices are randomized at each step.

\subsection{Genetic Algorithm}
Each chromosome encodes the $m$ team plans $X^{\mu}$ as a flat vector with the same structural layout as the SA state. The initial population is built via round-robin assignment, as in SA. The fitness function is
the negative of the SA energy function, so that maximizing fitness corresponds to minimizing makespan subject to the same soft constraints (coverage and energy). The remaining constraints \eqref{eq:con4} and
\eqref{eq:con5} are enforced as hard constraints by the chromosome representation and the crossover/mutation operators.

Crossover operates at the team-matrix block level: for each team $\mu$, the offspring inherits either parent~1's or parent~2's entire $X^{\mu}$ block with equal probability. This preserves intra-team tour structure rather than blending genes coordinate-wise. The mutation operator parallels the SA move operator, applying the same intra-team and cross-team neighborhood operations with a hyperparameter-tuned
probability per chromosome.

\end{document}